\documentclass[sigconf]{acmart}
\makeatletter
\def\@ACM@checkaffil{
    \if@ACM@instpresent\else
    \ClassWarningNoLine{\@classname}{No institution present for an affiliation}%
    \fi
    \if@ACM@citypresent\else
    \ClassWarningNoLine{\@classname}{No city present for an affiliation}%
    \fi
    \if@ACM@countrypresent\else
        \ClassWarningNoLine{\@classname}{No country present for an affiliation}%
    \fi
}
\makeatother

\usepackage{microtype}
\usepackage{graphicx}
\usepackage{subfigure}
\usepackage{booktabs} 

\usepackage{hyperref}

\usepackage[ruled,vlined]{algorithm2e}

\usepackage{amsfonts}
\usepackage{amsmath}
\usepackage[capitalize,noabbrev]{cleveref}

\theoremstyle{plain}
\newtheorem{theorem}{Theorem}[section]
\newtheorem{proposition}[theorem]{Proposition}
\newtheorem{lemma}[theorem]{Lemma}

\theoremstyle{definition}
\newtheorem{definition}[theorem]{Definition}

\theoremstyle{remark}

\usepackage[textsize=tiny]{todonotes}

\usepackage{bbm,bm} 
\usepackage{xspace,balance,tabularx,multirow}
\usepackage{enumitem}
\usepackage{bbding}
\usepackage{wrapfig}
\usepackage{mathtools}
\usepackage{caption}
\usepackage{subcaption}
\usepackage{hyperref}
\newcommand{\commentout}[1]{{}}

\newcommand{\wrt}{{\it w.r.t.}\xspace}
\newcommand{\ie}{{\it i.e.},\xspace}
\newcommand{\eg}{{\it e.g.},\xspace}

\def\header{\vspace{1mm} \noindent}

\newcommand{\algo}{SGOOD\xspace}

\newcommand{\energy}{Energy\xspace}
\newcommand{\msp}{MSP\xspace}
\newcommand{\odin}{ODIN\xspace}
\newcommand{\md}{MAH\xspace}
\newcommand{\graphde}{GraphDE\xspace}
\newcommand{\safe}{GNNSafe\xspace}
\newcommand{\good}{GOOD-D\xspace}
\newcommand{\ocgin}{OCGIN\xspace}
\newcommand{\glkd}{GLocalKD\xspace}
\newcommand{\ocgtl}{OCGTL\xspace}

\newcommand{\ngnn}{NGNN\xspace}
\newcommand{\gnnak}{$\text{GNN-AK}^{+}$\xspace}
\newcommand{\aagod}{AAGOD\xspace}
\newcommand{\enzymes}{ENZYMES\xspace}
\newcommand{\multi}{IMDB-M\xspace}
\newcommand{\bina}{IMDB-B\xspace}
\newcommand{\bace}{BACE\xspace}
\newcommand{\bbbp}{BBBP\xspace}
\newcommand{\drug}{DrugOOD\xspace}
\newcommand{\reddit}{REDDIT-12K\xspace}
\newcommand{\hiv}{HIV\xspace}

\newcommand{\auc}{AUROC\xspace}
\newcommand{\pr}{AUPR\xspace}
\newcommand{\fpr}{FPR95\xspace}
\newcommand{\acc}{ID ACC\xspace}

\newcommand{\Din}{D^{in}\xspace}

\newcommand{\Dtest}{D_{test}\xspace}

\newcommand{\SupGraph}{\mathcal{G}\xspace}
\newcommand{\subij}{g_{i,j}\xspace}

\AtBeginDocument{%
  }

\setcopyright{acmlicensed}
\copyrightyear{2018}
\acmYear{2018}
\acmDOI{XXXXXXX.XXXXXXX}

\acmConference[CIKM '24]{International Conference on
Information and Knowledge Management}{October 21 –- 25, 2024}{Boise, Idaho, USA}

\acmISBN{978-1-4503-XXXX-X/18/06}

\begin{document}

\title{\algo: Substructure-enhanced Graph-Level Out-of-Distribution Detection}

\author{Zhihao Ding}
\email{tommy-zh.ding@connect.polyu.hk}
\orcid{0000-0001-7778-6142}
\affiliation{%
  \institution{The Hong Kong Polytechnic University}
}

\author{Jieming Shi}
\authornote{Corresponding Author.}
\email{jieming.shi@polyu.edu.hk}
\orcid{0000-0002-0465-1551}
\affiliation{%
  \institution{The Hong Kong Polytechnic University}
}

\author{Shiqi Shen}
\email{shiqishen@tencent.com}
\orcid{}
\affiliation{%
  \institution{Tencent Inc.}
}

\author{Xuequn Shang}
\email{shang@nwpu.edu.cn}
\orcid{}
\affiliation{%
  \institution{The Northwestern Polytechnical University}
}

\author{Jiannong Cao}
\email{csjcao@comp.polyu.edu.hk}
\orcid{0000-0002-2725-2529}
\affiliation{%
  \institution{The Hong Kong Polytechnic University}
}

\author{Zhipeng Wang}
\email{markrocwang@tencent.com}
\orcid{}
\affiliation{%
  \institution{Tencent Inc.}
}

\author{Zhi Gong}
\email{davidgong@tencent.com}
\orcid{}
\affiliation{%
  \institution{Tencent Inc.}
}

\renewcommand{\shortauthors}{Ding et al.}

\begin{abstract}
Graph-level representation learning is important in a wide range of applications. Existing graph-level models are generally built on i.i.d. assumption for both training and testing graphs. However, in an open world,  models can encounter out-of-distribution (OOD) testing graphs that are from different distributions unknown during training. 
A trustworthy model should be able to detect OOD graphs to avoid unreliable predictions, while producing accurate in-distribution (ID) predictions.
To achieve this, we present \algo, a novel graph-level OOD detection framework. We find that substructure differences commonly exist between ID and OOD graphs, and  design \algo with a series of techniques to encode task-agnostic substructures for effective OOD detection.
Specifically, we build a super graph of substructures for every graph, and develop a two-level graph encoding pipeline that works on both original graphs and super graphs to obtain substructure-enhanced graph representations. 
We then devise substructure-preserving graph augmentation techniques to further capture more substructure semantics of ID graphs.
Extensive experiments against 11 competitors on numerous graph datasets demonstrate the superiority of \algo, often surpassing existing methods by a significant margin.  The code is available at \url{https://anonymous.4open.science/r/SGOOD-0958}.
\end{abstract} 





\maketitle

\section{Introduction}\label{sec:introduction}

Graphs are widely used to represent complex structured data, \eg chemical compounds, proteins, and social networks. 
Graph-level representation learning, which extracts meaningful representations of these graphs, is crucial for applications in biochemistry \cite{jiang2021could,rong2020self} and social network analysis \cite{dou2021user,shao2017efficient}.  
Existing graph-level learning models are based on the closed-world assumption, in which testing graphs encountered at deployment are drawn i.i.d. from the same distribution as the training graph data. 
However, in reality, the models are actually in an \textit{open world}, where test graphs can come from different, previously unseen distributions, making them out-of-distribution (OOD) \wrt in-distribution (ID) training graphs \cite{li2022ood,li2022graphde,yang2022learning}. 
Consequently, the models trained solely by ID graphs tend to make  incorrect  predictions on OOD data \citep{hendrycks2016baseline}, which raises reliability concerns in safety-critical applications, \eg drug discovery \citep{basile2019artificial}.
A trustworthy graph-level learning model should be capable of identifying OOD test graphs to avoid unreliable predictions.

While initial efforts have been made to explore graph-level OOD detection \cite{li2022graphde, liu2023good, guo2023data}, these methods primarily rely on message-passing graph neural networks (GNNs)~\cite{kipf2017semi,hamilton2017inductive} to first get node representations and then generate graph-level representations solely based on these nodes. Substructure patterns, which are high-level graph structures and contain rich graph semantics~\cite{zhang2021motif,yang2022learning}, have yet to be leveraged for graph-level OOD detection.
Intuitively, an OOD detector that can distinguish both node-level structures and substructure patterns would be more effective.
However, leveraging substructures for OOD detection is challenging due to the absence of OOD graphs during training, making it difficult to
determine in advance which substructure patterns should be learned for identifying OOD samples. 
Existing substructure learning methods, \eg hierarchical GNNs~\cite{ying2018hierarchical,lee2019self,gao2019graph} and subgraph GNNs~\cite{zhao2021stars,zhang2021nested}, typically learn \emph{task-specific} substructures that are tailored to discriminate between labeled classes in the training set. 
We argue that such task-specific substructures are not sufficient for OOD detection. Instead, the detector should be able to learn a diverse set of substructure patterns from ID training graphs, including \emph{task-agnostic} substructures not associated with specific classification tasks. 
As illustrated in Figure 1, the two ID graphs can be classified into different classes based on the presence of a task-specific triangle substructure (in blue). 
However, to identify the OOD graph, we need to compare its substructures to the task-agnostic ones such as the 6-node cycle (in green), to detect the irregular substructure (in orange). 
Therefore, effectively leveraging substructures for graph-level OOD detection requires considering task-agnostic information \cite{winkens2020contrastive}, a capability lacking in existing OOD detection methods.

\begin{table}[t!]

  \centering
\caption{The percentage of OOD graphs with substructures never appeared in ID graphs.}
\label{tab:ood}
\setlength{\tabcolsep}{2pt}
\vspace{-1mm}
\resizebox{0.4\textwidth}{!}{
    \begin{tabular}{ccccccc}
    \toprule
    Data & ENZYMES & IMDB-M & IMDB-B & BACE  & BBBP  & DrugOOD \\
    \midrule
    & 58.9\% & 14.0\% & 8.5\% & 50.0\% & 44.6\% & 77.3\% \\
    \bottomrule
    \end{tabular}
    }

  \vspace{-4mm}
\end{table} 

\begin{figure}
    \centering
    \includegraphics[width=0.7\columnwidth]{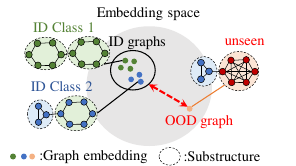}
    \vspace{-1mm}
    \caption{Substructure-enhanced graph-level OOD detection}
    \label{fig:illustration}
    \vspace{-5mm}
\end{figure}

In this paper, we develop \algo, a novel framework that {{explicitly} encodes task-agnostic substructures and their relationships} into effective representations for graph-level OOD detection. The design of \algo is supported  by empirical findings that demonstrate the crucial role of task-agnostic substructures in distinguishing between ID and OOD graphs. 
Given a dataset of graphs (see Table \ref{tab:datasets} for data statistics), we apply community detection ~\cite{clauset2004finding}
to extract task-agnostic substructures. The substructures are task-agnostic since the adopted community detection is independent of specific learning tasks. 
Then the percentage of OOD graphs with substructures that never appeared in ID graphs per dataset is reported in Table~\ref{tab:ood}. Observe that such percentage values are high, more than 44\% in 4 out of 6 datasets. The results validate that task-agnostic substructures can reveal differences between ID and OOD graphs.
As illustrated in Figure \ref{fig:illustration}, if a method can preserve the task-agnostic substructures of ID graphs into embeddings, OOD graphs with unseen substructures will have embeddings distant from those of ID graphs, making them easy to detect.

Therefore, we design a series of techniques in \algo to effectively encode task-agnostic substructures and generate  {substructure-enhanced graph representations} for graph-level OOD detection.
Specifically, we first build a {super graph} $\mathcal{G}_i$ {of substructures} for every graph $G_i$ to obtain task-agnostic substructures and their relationships. 
Then, a {two-level graph encoding pipeline} is designed to work on $G_i$ and $\mathcal{G}_i$ in sequence to learn expressive substructure-enhanced graph representations.
We prove that \algo with the pipeline is strictly more expressive than 1\&2-WL, which theoretically justifies the power of preserving substructure patterns for OOD detection.
{Moreover, to capture more information about task-agnostic substructures in training ID graphs,} we design {substructure-preserving graph augmentation} techniques, which utilize the super graph of substructures to ensure that the substructures in a graph are modified as a whole.
At test time, given a graph $G_i$ and its super graph $\mathcal{G}_i$, our OOD detector obtains the graph-level representations of both, which are then used for OOD score estimation.
Extensive experiments compare \algo with 11 baselines across 8 real-world graph datasets with various OOD types. \algo consistently outperforms existing methods, for example, achieving a 9.58\% absolute improvement in \auc over the runner-up baseline on the IMDB-M dataset.
In summary, our contributions are:
 \begin{itemize}[leftmargin=*]
\setlength\itemsep{1pt}
     \item We present \algo, a leading method that highlights the importance of task-agnostic substructures and effectively leverages them to enhance graph-level OOD detection.
     \item We design a two-level graph encoding pipeline by leveraging a super graph of substructures, empowering \algo to learn graph representations enhanced with substructures.
     \item  We further develop   substructure-preserving graph augmentations via super graphs of substructures, to  strengthen \algo's ability in distinguishing OOD graphs.
     \item Extensive experiments demonstrate the superiority of \algo for graph-level OOD detection, achieving significant improvements over existing methods across multiple datasets.
 \end{itemize}

\section{Preliminaries}

We consider graph-level classification, which aims to classify a collection of graphs into different classes. 
Let $G_i=(V_i, E_i)$ be a graph, where $V_i$ and   $E_i$  are node set and edge set, respectively. 
Let $\mathbf{x}_u \in \mathbb{R}^{c}$  denote the attribute vector of node $u \in V_i$ in graph $G_i$. 
Denote $\mathcal{X}$  as the in-distribution (ID) graph space and let   $\mathcal{Y}=\{1,2,...,C\}$ be the label space.
In   graph-level classification, the training set $\Din_{tr}=\{(G_i,y_i)\}^n_{i=1}$ is  drawn i.i.d. from the joint data distribution $P_{\mathcal{X}\mathcal{Y}}$. Every graph sample in $\Din_{tr}$ contains a graph $G_i$ with label $y_i$.
Let $f$ be a learning model trained by the training set $\Din_{tr}$, and $f$ is deployed to predict the label of a testing graph.

\header
\textbf{Graph-level Out-Of-Distribution Detection.} At test time, graph-level OOD detection can be treated as a task to decide whether a testing graph $G_{i} \text{ in testing set } \Dtest$  is from the ID $P_{\mathcal{X}}$ or from other irrelevant distributions (\ie OOD).
A typical way for OOD detection is to develop an OOD detector by leveraging the representations generated from the classification model $f$ that is trained via ID training graphs in $\Din_{tr}$. 
Specifically, the OOD detector has a scoring function $S(G_i)$ for every testing graph $G_i\in D_{test}$.
Testing graphs with low scores $S(G_i)$ are regarded as ID, while the graphs with high scores are OOD. As stated in \cite{ming2023exploit}, a score threshold $\lambda$ is typically set so that a high fraction of ID data (\eg 95\%) is correctly classified.

\begin{figure}[t!]
	\centering
		\includegraphics[width=0.48\textwidth]{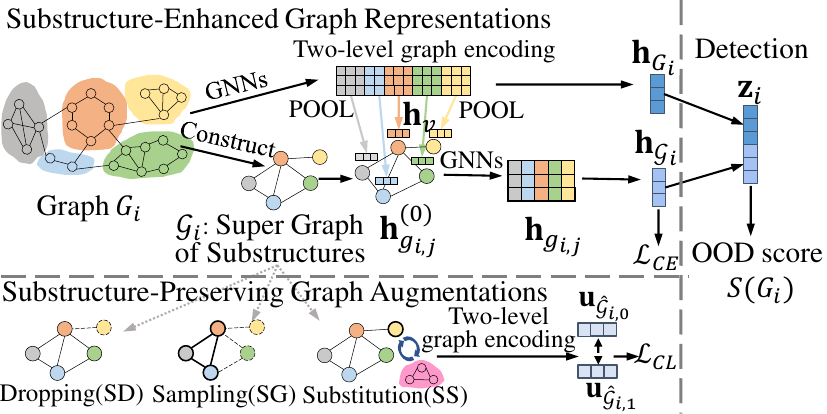}
  
  \vspace{-4mm}
 \caption{The \algo framework.}  
\label{fig:frame}   
\vspace{-5mm}
\end{figure}

\section{The \algo Method}\label{sec:method}

\header
\textbf{Solution Overview.}
The main goal of \algo is to effectively encode task-agnostic substructures and their relationships into representations for graph-level OOD detection. 
As illustrated in
Figure \ref{fig:frame}, \algo generates {substructure-enhanced graph representations} and further improves representation quality by {substructure-preserving graph augmentations}.
Given a graph $G_i$, {we first build its {super graph} $\SupGraph_i$ of task-agnostic substructures}, in which a super node represents a substructure in $G_i$ and edges connect super nodes by following the substructure connectivity in graph $G_i$.
A {two-level graph encoding pipeline} is designed over both $G_i$ and $\mathcal{G}_i$ for learning graph-level representations that are enhanced by substructures.
Furthermore, substructure-preserving graph augmentations are designed to preserve more information on task-agnostic substructures, making OOD graphs with unseen substructure patterns easier to detect.
Specifically, given a graph $G_i$, we augment it by first performing dropping, sampling, and substitution on its super graph $\SupGraph_i$ and then mapping the changes to  $G_i$ accordingly. 
In this way, the substructures in $G_i$ are modified as a whole. 
\algo is trained using a combination of classification loss $\mathcal{L}_{CE}$ and contrastive loss $\mathcal{L}_{CL}$.
In test time, given a test graph $G_i$, we first obtain graph-level representations of both $G_i$ and its super graph $\SupGraph_i$, concatenate and normalize the representations to compute the OOD score $S(G_i)$.

\subsection{Substructure-Enhanced Graph Encoding}\label{sec::substructure}

Given a graph $G_i$, we first describe how to get its super graph $\mathcal{G}_i$ of task-agnostic substructures, and then present a two-level graph encoding pipeline to generate substructure-enhanced graph representations. 

As \algo utilizes task-agnostic substructures, we treat substructure detection as a pre-processing step, and it is not our focus on how to detect substructures. There exist off-the-shelf methods \cite{dhillon2007weighted, cordasco2010community} to detect task-agnostic substructures. By default, we use modularity-based community detection \cite{clauset2004finding}. We also test other subgraph detection methods and find that the modularity-based substructures are effective in \algo, as validated in Table \ref{tab:community}.

\header
\textbf{Super Graph  of Substructures.}
Let a substructure $\subij$ of graph $G_i$ be a connected subgraph of $G_i$.
Specifically, a subgraph
$\subij=({V}_{i,j}, {E}_{i,j})$ is a substructure of $G_i=(V_i, E_i)$ iff $V_{i,j} \subseteq  {V_i}$, $E_{i,j} \subseteq E_i$, and  $\subij$ is connected.
The substructures $\{\subij\}_{j=1}^{n_{i}}$ of a graph $G_i$ satisfy the following properties: (i) the node sets of substructures are non-overlapping, (ii) the union of the nodes in all substructures is the node set of $G_i$, and (iii) every substructure is a  connected subgraph of $G_i$. 
Then we construct the super graph $\mathcal{G}_i$   by regarding every substructure $\subij$ as a super node in $\mathcal{G}_i$, and connect super nodes by inserting edges via Definition \ref{def:supgraph}.
Super graph $\mathcal{G}_{i}$ is a higher-order view depicting the relationships between the substructures of a graph $G$. We also add self-loops in super graph $\mathcal{G}_{i}$.

\begin{definition}[A Super Graph of Substructures] \label{def:supgraph}A super graph of substructures constructed from the input graph $G_{i}=(V_i,E_i)$ is denoted as $\mathcal{G}_i=(\mathcal{V}_{i}, \mathcal{E}_{i})$, where every super node $\subij$ in node set  $\mathcal{V}_{i} = \{\subij\}_{j=1}^{n_{i}}$ represents a substructure of $G_{i}$, and every edge in $\mathcal{E}_i$ connecting two super nodes, and the edge set $\mathcal{E}_{i}=\{(\subij,g_{i,k})|\exists u\in V_{i,j} \land v\in V_{i,k}, (u,v)\in E_i\}$.

\end{definition}

\header
\textbf{Two-level Graph Encoding.} Given a graph $G_i$ and its super graph $\mathcal{G}_i$, we present a two-level graph encoding pipeline, as shown in  Figure \ref{fig:frame}.
The idea is that, in addition to learning representations over $G_i$, we further utilize the super graph $\mathcal{G}_i$ to encode substructure information into graph representations, to better preserve distinguishable substructure patterns for effective OOD detection.
The two-level graph encoding first adopts GNNs to learn node representations with initial features over graph $G_i$. For every node $v\in V_i$, its representation $\mathbf{h}_v^{(l+1)}$ at $(l+1)$-layer is obtained by Eq. \eqref{eq:node_mp}.
{Different GNNs have different aggregation and combination functions $f_{\text{AGGR}}$, $f_{\text{COMB}}$.} By default, we adopt Graph Isomorphism Network (GIN) \cite{xu2018powerful} as the backbone.
The GIN for $G_i$ has $L_1$ layers. 
\begin{equation}\label{eq:node_mp} 
    \mathbf{h}^{(l+1)}_{v} = f_{\text{COMB}}^{(l+1)}(\mathbf{h}^{(l)}_{v}, f_{\text{AGGR}}^{(l+1)}(\mathbf{h}^{(l)}_{u}, u \in N_{G_i}(v))), 
\end{equation}
where $\mathbf{h}^{(l)}_{v} \in \mathbb{R}^{d}$ is the intermediate representation of node $v$ from the $l$-th layer GNNs with hidden dimension $d$, $f_{\text{AGGR}}^{(l+1)}$ is the function that aggregates node features from $v$'s neighborhood ${N}_{G_i}(v)$ in graph $G_i$, $f_{\text{COMB}}^{(l+1)}$ is the function that updates node $v$'s representation by combining the representations of its neighbors with its own, and initially $\mathbf{h}^{(0)}_{v}=\mathbf{x}_v$.

Next, we obtain the {representations of substructures} $\subij$ in $G_i$ by leveraging the node representations above. 
As shown in Eq.~\eqref{eq:node_pool}, given a node $v$, we first concatenate all representations $\mathbf{h}_v^{(l)}$ for  $l=1,...,L_1$ {using $f_{\text{CAT}}$}  to get  $\mathbf{h}_v$ that preserves multi-scale semantics.
Then,  for a substructure $\subij$ of graph $G_i$, we obtain the substructure representation $\mathbf{h}^{(0)}_{\subij}$  by integrating  $\mathbf{h}_v$ of all $v$ in $\subij$ via DeepSet pooling \cite{zhang2019deep} $f_{\text{POOL}}$ in Eq. \eqref{eq:node_pool}. 
\begin{equation}\label{eq:node_pool} 
     \mathbf{h}^{(0)}_{\subij} = f_{\text{POOL}}(\{\mathbf{h}_{v}|v\in {V}_{i,j} \}), \mathbf{h}_{v} = f_{\text{CAT}}(\{\mathbf{h}^{(l)}_{v}\}_{l=1}^{L_1})
\end{equation}
Note that $\mathbf{h}^{(0)}_{\subij}$ only considers the nodes inside substructure $\subij$ and the original graph topology $G_i$.
To further consider the relationships depicted in the super graph $\mathcal{G}_i$ of substructures, we regard $\mathbf{h}^{(0)}_{\subij}$ as the initial features of super node $\subij$ in $\mathcal{G}_i$, and employ a $L_2$-layer GIN over $\mathcal{G}_i$  to learn substructure-enhanced graph representations by Eq. \eqref{eq:sub_mp} and \eqref{eq:sub_pool}.
\begin{equation}\label{eq:sub_mp}
    \mathbf{h}^{(l+1)}_{g_{i,j}} = f_{\text{COMB}}^{(l+1)}(\mathbf{h}^{(l)}_{g_{i,j}}, f_{\text{AGGR}}^{(l+1)}(\mathbf{h}^{(l)}_{g_{i,k}}, g_{i,k} \in {N}_{\mathcal{G}_i}(g_{i,j}))),  
\end{equation}
where ${N}_{\mathcal{G}_i}(g_{i,j})$ is the neighbors of super node $\subij$ in $\mathcal{G}_i$.

Lastly,  in Eq. \eqref{eq:sub_pool}, we get the final representation $\mathbf{h}_{\subij}$ of every super node $\subij$ by concatenating the representation of $\subij$ in every layer of the $L_2$-layer GIN, and obtain the graph representation $\mathbf{h}_{\mathcal{G}_i}$ by readout  $f_{\text{OUT}}$ that is sum pooling.   
\begin{equation}\label{eq:sub_pool}
    \mathbf{h}_{\mathcal{G}_i} = f_{\text{OUT}}(\{\mathbf{h}_{g_{i,j}}|g_{i,j}\in \mathcal{V}_{i} \}), 
\mathbf{h}_{g_{i,j}} = f_{\text{CAT}}(\{\mathbf{h}^{(l)}_{g_{i,j}}\}_{l=0}^{L_2})
\end{equation}

Remark that the representation  $\mathbf{h}_{\mathcal{G}_i}$ of super graph $\mathcal{G}_i$ of graph $G_i$ is used  to train the loss $\mathcal{L}_{CE}$ for classification.
Meanwhile, as explained shortly, for OOD detection during testing, we further consider another representation of graph $G_i$ obtained by aggregating node representations. 
Existing studies, such as hierarchical pooling \cite{ying2018hierarchical,lee2019self,gao2019graph} and subgraph GNNs \cite{zhang2021nested,zhao2021stars}, learn task-specific substructures for the graph classification task.
On the other hand, we leverage task-agnostic substructures.
We conduct experiments to demonstrate that our \algo is more effective than these methods for graph-level OOD detection.

\subsection{Substructure-Preserving  Augmentations}\label{sec::augmentations} 
{Intuitively, if more information about task-agnostic substructures in training ID graphs is preserved, it is easier to distinguish OOD graphs with unseen substructure patterns.}
To achieve this, we design substructure-preserving graph augmentations by leveraging the super graph $\mathcal{G}_i$ of graph $G_i$, to improve the performance further.
However, it is challenging to achieve this. 
Substructures with subtle differences have different semantics. It is important to keep the substructures of a graph intact while performing augmentations. 
Common augmentation techniques like edge permutation and node dropping directly on graphs $G_i$ may unexpectedly destroy meaningful substructures, and hamper OOD detection effectiveness. 

To tackle the issue, we first perform augmentations on the super graph  $\mathcal{G}_i$ by regarding substructures as atomic nodes, and then map the augmentations to the original graph $G_i$ with modifications over substructures as a whole.
Specifically, we propose three substructure-level graph augmentations below, namely \textit{substructure dropping (SD), super graph sampling (SG), and substructure substitution (SS).} 
The default augmentation ratio is set to 0.3.
\vspace{-1mm}
\begin{itemize}[leftmargin=*]  \setlength\itemsep{0em}
    \item \textit{Substructure Dropping (SD).} Given a graph $G_i$ with its super graph $\mathcal{G}_i$, a fraction of super nodes in $\mathcal{G}_i$ (\ie the corresponding substructures in $G_i$) are discarded randomly. Remark that selected substructures are dropped as a whole.
    \item \textit{Super Graph Sampling (SG).} In the super graph $\mathcal{G}_i$, we start from a random node, sample a fixed-size subgraph in $\mathcal{G}_i$, and drop the rest nodes and edges. The changes are mapped to $G_i$ accordingly. Depth-first search is chosen as the sampling strategy \cite{you2020graph}.
    \item \textit{Substructure Substitution (SS).} Given a graph $G_i$ in class $c$ with super graph $\mathcal{G}_i$ of substructures, we randomly substitute a fraction of nodes in $\mathcal{G}_i$ (\ie substructures in $G_i$) with other substructures from the graphs of the same class $c$. To avoid drastic semantic change of the whole graph, only super nodes with degree one (excluding self-loops) in $\mathcal{G}_i$ take part in the substitution. 
\end{itemize}

\subsection{Model Training and OOD Scoring}\label{sec::training}
\subsubsection{Two-stage model training}
We adopt a cross-entropy loss $\mathcal{L}_{CE}$ for classification. 
After getting representation $\mathbf{h}_{\mathcal{G}_i}$ for the super graph $\mathcal{G}_i$ of graph $G_i$, we apply a linear transformation to get prediction logits 
$\widehat{y}_{i}$, evaluated against the ground-truth class label $y_i$ to get $\mathcal{L}_{CE}$ by Eq. \eqref{eq:lossce} for a mini-batch of $B$ training graphs.
\begin{equation}\label{eq:lossce}
\mathcal{L}_{CE}=-\frac{1}{B} \textstyle\sum_{i=1}^{B} \textstyle\sum_{c=1}^{C} \mathbbm{1}(y_{i}=c) \log \left(\hat{y}_{i,c}\right)
\end{equation}
Then we adopt the substructure-preserving graph augmentations in Section \ref{sec::augmentations} to get contrastive loss  $\mathcal{L}_{CL}$.
Specifically, given a mini-batch of $B$ training graphs $\{G_{i}\}_{i=1}^{B}$ and their super graphs $\{\mathcal{G}_{i}\}_{i=1}^{B}$, we transform the super graphs to get $\widehat{\mathcal{G}}_{i,0}=\mathcal{T}_{0}(\mathcal{G}_{i})$ and  $\widehat{\mathcal{G}}_{i,1} =\mathcal{T}_{1}(\mathcal{G}_{i})$, where $\mathcal{T}_{0}$ and $\mathcal{T}_{1}$ are two augmentations chosen among $\mathcal{A}=\{$\textit{{I}, {SD}, {SG}, {SS}}$\}$, where \textit{I} indicates no augmentation. 
Graph $G_{i}$ is transformed accordingly via $\mathcal{T}_{0}$ and $\mathcal{T}_{1}$ to obtain $\widehat{G}_{i,0}$ and $\widehat{G}_{i,1}$ respectively.  Then, the representations $\mathbf{h}_{\widehat{\mathcal{G}}_{i,0}}$ and $\mathbf{h}_{\widehat{\mathcal{G}}_{i,1}}$ of the two augmented super graphs can be calculated by applying Eq.\eqref{eq:node_mp}-\eqref{eq:sub_pool}. 
We transform $\mathbf{h}_{\widehat{\mathcal{G}}_{i,0}}$ and $\mathbf{h}_{\widehat{\mathcal{G}}_{i,1}}$ by  a shared projection head  $\psi(\cdot)$, which is  a 2-layer MLP, followed by $l_{2}$-normalization,  to obtain $\mathbf{u}_{\widehat{\mathcal{G}}_{i,0}} = \psi(\mathbf{h}_{\widehat{\mathcal{G}}_{i,0}})/||\psi(\mathbf{h}_{\widehat{\mathcal{G}}_{i,0}})||$ and $\mathbf{u}_{\widehat{\mathcal{G}}_{i,1}} = \psi(\mathbf{h}_{\widehat{\mathcal{G}}_{i,1}})/||\psi(\mathbf{h}_{\widehat{\mathcal{G}}_{i,1}})||$, respectively.  We get $\mathcal{L}_{CL}$ by 
\begin{equation}\label{eq:lossCL}
\small
    \mathcal{L}_{CL} = \frac{1}{2B}\sum_{i=1}^{B}\sum_{a\in\{0,1\}} -\log\frac{\exp{(\mathbf{u}_{\widehat{\mathcal{G}}_{i,a}}^{\intercal}\mathbf{u}_{\widehat{\mathcal{G}}_{i,1-a}}/\tau)}}{\sum_{j=1}^{B}\mathbbm{1}(j\neq i)\sum_{k\in\{0,1\}}\exp{(\mathbf{u}_{\widehat{\mathcal{G}}_{i,a}}^{\intercal}\mathbf{u}_{\widehat{\mathcal{G}}_{j,1-k}}/\tau)}}, 
\end{equation}
where $\tau$ is a temperature parameter.

The overall training loss is the combination of $\mathcal{L}_{CE}$ and $\mathcal{L}_{CL}$, weighted by $\alpha$:
\begin{equation}
\label{eq:loss}
\mathcal{L}=\mathcal{L}_{CE} + \alpha\mathcal{L}_{CL}.
\end{equation} 
The training procedure of \algo consists of two stages. In the first pre-training stage, the parameters are {solely} updated by minimizing $\mathcal{L}_{CL}$ for $T_{PT}$ epochs. In the second stage, the parameters are fine-tuned under the {combined} overall loss $\mathcal{L}$ for $T_{FT}$ epochs. 
This training procedure achieves better performance than directly training $\mathcal{L}$, as shown  in Figure~\ref{fig:epochs}  {when pretraining $T_{PT}$ is 0.}

\subsubsection{Graph-level OOD  scoring}\label{sec::detection}
Recall that the main goal of \algo is to let the representations of ID data and OOD data to be distant from each other. Given a testing graph $G_{i}\in D_{test}$, we use the standard Mahalanobis distance \cite{lee2018simple}
to quantify its OOD score.  If $G_i$ is with large  Mahalanobis distance from the ID training data, it  tends to be OOD. 

As shown in Figure \ref{fig:frame}, in addition to the representation $\mathbf{h}_{\mathcal{G}_i}$ of the super graph $\mathcal{G}_i$ of a testing graph $G_i$, \algo also aggregates the node representations of $G_i$ to get
$\mathbf{h}_{G_{i}}=f_{\text{OUT}}(\{\mathbf{h}_{v}|v\in {V}_{i}\})$.
Representations $\mathbf{h}_{\mathcal{G}_i}$ and  $\mathbf{h}_{G_{i}}$ are concatenated together to estimate the OOD score $S(G_i)$:

\begin{equation}\label{eq:md}
\small
    S(G_{i}) = \textstyle\max_{c\in [C]}(\mathbf{z}_{i}-\bm{\mu}_{c})^{\intercal}\widehat{ \Sigma}^{-1}(\mathbf{z}_{i}-\bm{\mu}_{c}),
    \mathbf{z}_{i} = \frac{f_{\text{CAT}}(\mathbf{h}_{G_{i}},\mathbf{h}_{\mathcal{G}_{i}})}{||f_{\text{CAT}}(\mathbf{h}_{G_{i}},\mathbf{h}_{\mathcal{G}_{i}})||_{2}},
\end{equation}
where $[C] = \{1,\ldots,C\}$, 
$\bm{\mu}_{c}$ and $\widehat{ \Sigma}$  are the  class centroid for class $c$ and covariance matrix of training ID graphs, respectively.

\subsection{Theoretical Analysis}
We analyze the expressive power of \algo (using message-passing GNNs as the backbone) in comparison to 1\&2-WL, a key tool for evaluating the expressivity of GNNs~\cite{xu2018powerful}. In Proposition~\ref{prop:power}, we demonstrate that \algo is more expressive than 1\&2-WL, enabling it to distinguish structural patterns beyond the capability of 1\&2-WL, and consequently, message-passing GNNs. The analysis, together with our empirical findings in Section \ref{sec:introduction}, explains the power of \algo for graph-level OOD detection.
\begin{proposition}\label{prop:power}When the GNNs adopted in \algo are  with sufficient number of layers, and the $f_{\text{POOL}}$ function in Eq.\eqref{eq:node_pool} and $f_{\text{OUT}}$ function in Eq.\eqref{eq:sub_pool} are injective, then \algo is strictly more expressive than 1\&2-WL. 
\end{proposition}
\begin{proof}
We first prove that \algo is at least as powerful as 1\&2-WL in Lemma \ref{lm:1}. Then, we prove that \algo can distinguish 2-regular graphs that 1\&2-WL  cannot distinguish in Lemma \ref{lm:2}. Combining these two Lemmas, we prove that \algo is strictly more expressive than  1\&2-WL.
\end{proof}
\begin{lemma}\label{lm:1}
For graphs $G_1=(V_1, E_1)$ and $G_2=(V_2, E_2)$  identified as non-isomorphic by  1\&2-WL, 
\algo projects them into  different representations $\mathbf{h}_{\mathcal{G}_1}$ and $\mathbf{h}_{\mathcal{G}_2}$ in Eq. \eqref{eq:sub_pool}.
\end{lemma}
\begin{proof}
Let $\mathcal{H}_{1}^{G}=\{\mathbf{h}_{v}|v\in {V}_{1} \}$ and $\mathcal{H}_{2}^{G}=\{\mathbf{h}_{v}|v\in {V}_{2} \}$ be the  multisets of node representations of $G_1$ and $G_2$ generated by GIN in Eq. \eqref{eq:node_pool}, respectively. Let $\mathcal{G}_1=(\mathcal{V}_{1}, \mathcal{E}_{1})$ and $\mathcal{G}_2=(\mathcal{V}_{2}, \mathcal{E}_{2})$ be the super graphs of $G_1$ and $G_2$ respectively.  We consider two cases: (1) $|\mathcal{V}_{1}| \neq |\mathcal{V}_{2}|$, (2) $|\mathcal{V}_{1}| = |\mathcal{V}_{2}|$.

For case (1), $\mathcal{G}_1$ and $\mathcal{G}_2$ are two graphs with different number of nodes. Thus, $\mathcal{G}_1$ and $\mathcal{G}_2$ can be easily determined as non-isomorphic by 1\&2-WL. It is  proved that GIN with sufficient number of layers and all injective functions is as powerful as 1\&2-WL \cite{xu2018powerful}. As GIN is adopted in \algo as GNN backbone  with sufficient number of layers and $f_{\text{OUT}}$ function in Eq.\eqref{eq:sub_pool} is injective,    representations $\mathbf{h}_{\mathcal{G}_1}$ and 
$\mathbf{h}_{\mathcal{G}_2}$ generated by \algo are different.

For case (2), let $|\mathcal{V}_{1}| = |\mathcal{V}_{2}| = K$, $\mathcal{H}_{1}^{\mathcal{G}}= \{\mathbf{h}^{(0)}_{g_{1,j}}|g_{1,j}\in \mathcal{V}_{1}\}$ and $\mathcal{H}_{2}^{\mathcal{G}}= \{\mathbf{h}^{(0)}_{g_{2,j}}|g_{2,j}\in \mathcal{V}_{2}\}$ be the multisets of initial node representations of $\mathcal{G}_1$ and $\mathcal{G}_2$ calculated by Eq.\eqref{eq:node_pool}, respectively. 
Using GIN with sufficient number of layers, we get $\mathcal{H}_{1}^{G} \neq \mathcal{H}_{2}^{G}$ \cite{xu2018powerful}. As stated in Section \ref{sec::substructure}, the substructures $\{\subij\}_{j=1}^{n_{i}}$ of a graph $G_i$ satisfy the following properties: (i) the substructures are non-overlapping, (ii) the union of nodes in all substructures is the node set of $G_i$. Thus, $\{\{\mathbf{h}_{v}|v\in {g}_{1,j}\}\}_{j=1}^{K}$ (resp. $\{\{\mathbf{h}_{v}|v\in {g}_{2,j}\}\}_{j=1}^{K}$) is a partition of $\mathcal{H}_{1}^{G}$ (resp. $\mathcal{H}_{2}^{G}$). Then, we have
$\{\{\mathbf{h}_{v}|v\in {g}_{1,j}\}\}_{j=1}^{K} \neq \{\{\mathbf{h}_{v}|v\in {g}_{2,j}\}\}_{j=1}^{K}$. As $f_{\text{POOL}}$ function in Eq.\eqref{eq:node_pool} is injective, we have $\{f_{\text{POOL}}(\{\mathbf{h}_{v}|v\in {g}_{1,j}\})\}_{j=1}^{K} \neq \{f_{\text{POOL}}(\{\mathbf{h}_{v}|v\in {g}_{2,j}\})\}_{j=1}^{K}$, that is $\mathcal{H}_{1}^{\mathcal{G}} \neq \mathcal{H}_{2}^{\mathcal{G}}$. As GIN and $f_{\text{OUT}}$ function in Eq.\eqref{eq:sub_pool} are both injective, we derive that  $\mathbf{h}_{\mathcal{G}_1}$ and 
$\mathbf{h}_{\mathcal{G}_2}$ generated on  $\mathcal{H}_{1}^{\mathcal{G}}$ and $\mathcal{H}_{2}^{\mathcal{G}}$ are different.

Combining cases (1) and (2), we prove Lemma \ref{lm:1}.
\end{proof}
\vspace{-2pt}
Next, we   prove that  \algo can distinguish $n$-node 2-regular graphs that 1\&2-WL cannot distinguish in Lemma \ref{lm:2}. Before that, we first give the definition of 2-regular graphs. Note that we only consider undirected graphs in this paper.
\begin{definition}[2-regular graph] A graph is said to be regular of degree 2 if all local degrees are  2. 
\end{definition}

\begin{lemma}\label{lm:2}
Given two non-isomorphic $n$-node 2-regular graphs $G_1$ and $G_2$ that 1\&2-WL cannot distinguish, \algo projects them into  different graph representations $\mathbf{h}_{\mathcal{G}_1}$ and $\mathbf{h}_{\mathcal{G}_2}$in Eq. \eqref{eq:sub_pool}.
\end{lemma}
\begin{proof}
Based on the definition of a 2-regular graph, we can say that  $G_1$ and $G_2$  consist of one or more disconnected cycles. Let $r_1$ and $r_2$ be the number of cycles in  $G_1$ and $G_2$. Let $\mathcal{G}_1=(\mathcal{V}_{1}, \mathcal{E}_{1})$ and $\mathcal{G}_2=(\mathcal{V}_{2}, \mathcal{E}_{2})$ be the constructed super graphs of $G_1$ and $G_2$. We consider two cases: (1) $r_1 \neq 1  \wedge
 r_2 \neq 1$, (2) $(r_1=1 \wedge r_2\neq 1)\vee(r_1\neq1 \wedge r_2= 1)$.
 
For case (1), $G_1$ and $G_2$ consist of disconnected circles. 
As $\mathcal{G}_1$ and $\mathcal{G}_2$ are constructed by modularity-based community detection method \cite{clauset2004finding}, in Lemma 3.4 of \cite{brandes2007modularity}, for graphs of disconnected circles, there is always a clustering with maximum
modularity, in which each cluster consists of a connected subgraph. As a result, $\forall g_{1,j} \in \mathcal{V}_{1}$ is a circle in $G_1$, and $|\mathcal{V}_{1}|=r_1$. Similarly,  $\forall g_{2,j} \in \mathcal{V}_{2}$ is  a circle in $G_2$, and $|\mathcal{V}_{2}|=r_2$. Let $\mathcal{N}_{1} = \{|V_{1,j}|\}_{j=1}^{|\mathcal{V}_{1}|}$ and $\mathcal{N}_{2} = \{|V_{2,j}|\}_{j=1}^{|\mathcal{V}_{2}|}$.
Since $G_1$ and $G_2$ are non-isomorphic, we have $\exists n_{1,j} \in \mathcal{N}_{1}: \forall n_{2,j} \in \mathcal{N}_{2}, n_{1,j} \neq n_{2,j}$. As a result, we have $\mathcal{N}_{1} \neq \mathcal{N}_{2}$. Then, we have
$\{\{\mathbf{h}_{v}|v\in {V}_{1,j}\}\}_{j=1}^{|\mathcal{V}_{1}|} \neq \{\{\mathbf{h}_{v}|v\in {V}_{2,j}\}\}_{j=1}^{|\mathcal{V}_{2}|}$. As $f_{\text{POOL}}$ function in Eq.\eqref{eq:node_pool} is injective, we have $\mathcal{H}_{1}^{\mathcal{G}}=\{f_{\text{POOL}}(\{\mathbf{h}_{v}|v\in {g}_{1,j}\})\}_{j=1}^{|\mathcal{V}_{1}|}$ , $\mathcal{H}_{2}^{\mathcal{G}}= \{f_{\text{POOL}}(\{\mathbf{h}_{v}|v\in {g}_{2,j}\})\}_{j=1}^{|\mathcal{V}_{2}|}$, and $\mathcal{H}_{1}^{\mathcal{G}} \neq \mathcal{H}_{2}^{\mathcal{G}}$. As shown in the proof of Lemma~\ref{lm:1}, we have  $\mathbf{h}_{\mathcal{G}_1}$ and 
$\mathbf{h}_{\mathcal{G}_2}$ generated on  $\mathcal{H}_{1}^{\mathcal{G}}$ and $\mathcal{H}_{2}^{\mathcal{G}}$ are different.

For case (2), we consider $r_1=1 \wedge r_2\neq 1$, and the proof when $r_2=1 \wedge r_1\neq 1$  is similar. $G_1$  consists of one single circle, and $G_2$  consists of $r_2$ disconnected circles. For $G_2$ and $\mathcal{G}_2$, $\forall g_{2,j} \in \mathcal{V}_{2}$ is  a circle in $G_2$, and $|\mathcal{V}_{2}|=r_2$
following the conclusion in case (1). For $G_1$ and $\mathcal{G}_1$, we consider two cases: (i) $|\mathcal{V}_{1}|=r_1=1$, and (ii) $|\mathcal{V}_{1}|>1$. For case (i), $\mathcal{V}_{1}=\{g_{1,1}\}$, where $g_{1,1}=G_1$. Let $\mathcal{N}_{1} = \{|V_{1,j}|\}_{j=1}^{|\mathcal{V}_{1}|}=\{|V_{1,1}|\}$ and $\mathcal{N}_{2} = \{|V_{2,j}|\}_{j=1}^{|\mathcal{V}_{2}|}$. As $|\mathcal{N}_{1}| \neq |\mathcal{N}_{2}|$, we have $\mathcal{N}_{1} \neq \mathcal{N}_{2}$. Similar to case (1), we conclude that graph representations $\mathbf{h}_{\mathcal{G}_1}$ and 
$\mathbf{h}_{\mathcal{G}_2}$ generated on  $\mathcal{H}_{1}^{\mathcal{G}}$ and $\mathcal{H}_{2}^{\mathcal{G}}$ are different. For case (ii), $\mathcal{V}_{1}=\{g_{1,j}\}_{j=1}^{|\mathcal{V}_{1}|}$, where $\forall g_{1,j} \in \mathcal{V}_{1}$ is a chain and two nearby chain are connected in $\mathcal{G}_1$. In other words, $\mathcal{G}_1$ is a $|\mathcal{V}_{1}|$-circle while $\mathcal{G}_2$ consists of $|\mathcal{V}_{2}|$ isolated nodes. Thus, $\mathcal{G}_1$ and $\mathcal{G}_2$ can be distinguished as non-isomorphic by 1\&2-WL.  By \cite{xu2018powerful}, when we encode $\mathcal{G}_1$ and $\mathcal{G}_2$ by Eq. \eqref{eq:sub_mp} with sufficient layers of GIN, and generate  $\mathbf{h}_{\mathcal{G}_1}$ and 
$\mathbf{h}_{\mathcal{G}_2}$ by   Eq. \eqref{eq:sub_pool}, where $f_{\text{OUT}}$ is injective, $\mathbf{h}_{\mathcal{G}_1}$ and 
$\mathbf{h}_{\mathcal{G}_2}$ are different. Combining cases (i) and (ii), we prove  \algo generates different $\mathbf{h}_{\mathcal{G}_1}$ and 
$\mathbf{h}_{\mathcal{G}_2}$ for $G_1$ and $G_2$ in case (2).

Combining  cases (1) and (2), we prove Lemma \ref{lm:2}.
\end{proof}

\vspace{5mm}

\begin{table}[!t]
  \centering
  \caption{Data Statistics.}
  \vspace{-3mm}
  \small
  \renewcommand{\arraystretch}{0.92} 
\setlength{\tabcolsep}{1.5pt}
\resizebox{0.99\linewidth}{!}{
    \begin{tabular}{lcccccc}
    \toprule
    Dataset  & OOD Type & \# Class & \# ID Train & \# ID Val & \# ID Test & \# OOD Test \\
    \midrule
    \enzymes\cite{morris2020tudataset}  & Unseen Classes &6     & 480   & 60    & 60    & 60 \\
    \multi\cite{morris2020tudataset}  & Unseen Classes & 3     & 1200  & 150   & 150   & 150 \\
    \bina\cite{morris2020tudataset} & Unseen Classes & 2     & 800   & 100   & 100   & 100 \\
    \reddit\cite{yanardag2015deep} & Unseen Classes & 11    & 6997  & 875   & 875   & 875 \\
    \bace\cite{wu2018moleculenet} & Scaffold   & 2     & 968   & 121   & 121   & 121 \\
    \bbbp\cite{wu2018moleculenet}   & Scaffold & 2     & 1303  & 164   & 164   & 164 \\
    \drug\cite{ji2022drugood}  & Protein Target & 2     & 800   & 100   & 100   & 100 \\
\hiv\cite{wu2018moleculenet}   & Scaffold & 2     & 26319 & 3291  & 3291  & 3291 \\
    \bottomrule
    \end{tabular}
  \label{tab:datasets}
 }
 \vspace{-3mm}
\end{table}

\section{Experiments}\label{sec:experiments}
We evaluate \algo in graph-level OOD detection against 11 baseline methods across 8 real-world datasets.

\begin{table*}[!t]
  \centering
  \caption{Overall OOD detection performance by \auc, \pr, and \fpr in percentage \% (mean ± std).  $\uparrow$ indicates larger values are better and vice versa.  \textbf{Bold}: best. \underline{Underline}: runner-up.}
    \vspace{-2mm}
  \resizebox{0.86\textwidth}{!}{
  \setlength{\tabcolsep}{1pt}
  \renewcommand{\arraystretch}{0.92}
  \hspace{-4mm}
    \begin{tabular}{ccll|ccc|ccc|cll}
    \toprule
    \multirow{2}[2]{*}{Method} & \multicolumn{3}{c|}{ENZYMES} & \multicolumn{3}{c|}{IMDB-M} & \multicolumn{3}{c|}{IMDB-B} & \multicolumn{3}{c}{\reddit} \\
          & AUROC$\uparrow$ & \multicolumn{1}{c}{AUPR$\uparrow$} & \multicolumn{1}{c|}{FPR95$\downarrow$} & AUROC$\uparrow$ & AUPR$\uparrow$ & FPR95$\downarrow$ & AUROC$\uparrow$ & AUPR$\uparrow$ & FPR95$\downarrow$ & AUROC$\uparrow$ & \multicolumn{1}{c}{AUPR$\uparrow$} & \multicolumn{1}{c}{FPR95$\downarrow$} \\
    \midrule
    MSP   & 61.34\scriptsize{±3.79} & \multicolumn{1}{c}{61.65\scriptsize{±6.64}} & \multicolumn{1}{c|}{89.67\scriptsize{±2.26}} & 42.75\scriptsize{±1.52} & 51.04\scriptsize{±1.93} & 95.73\scriptsize{±1.63} & 58.13\scriptsize{±2.31} & 59.63\scriptsize{±1.22} & 91.40\scriptsize{±4.16} & 50.63\scriptsize{±0.87} & 48.60\scriptsize{±1.08} & \multicolumn{1}{l}{95.95\scriptsize{±1.25}} \\
    Energy & 54.69\scriptsize{±9.18} & \multicolumn{1}{c}{56.90\scriptsize{±8.85}} & \multicolumn{1}{c|}{89.33\scriptsize{±3.55}} & 24.50\scriptsize{±19.73} & 37.26\scriptsize{±11.78} & 96.40\scriptsize{±2.25} & 49.58\scriptsize{±17.76} & 59.03\scriptsize{±13.06} & 92.80\scriptsize{±3.55} & 55.10\scriptsize{±0.48} & 56.52\scriptsize{±0.78} & \multicolumn{1}{l}{97.19\scriptsize{±0.58}} \\
    ODIN  & 63.70\scriptsize{±2.70} & \multicolumn{1}{c}{65.72\scriptsize{±4.77}} & \multicolumn{1}{c|}{92.66\scriptsize{±3.26}} & 40.12\scriptsize{±2.96} & 50.08\scriptsize{±2.44} & 96.66\scriptsize{±1.03} & 58.25\scriptsize{±2.94} & 61.36\scriptsize{±0.49} & 92.20\scriptsize{±2.92} & 51.74\scriptsize{±2.03} & 54.53\scriptsize{±1.26} & \multicolumn{1}{l}{96.45\scriptsize{±0.73}} \\
    \md & 67.37\scriptsize{±3.67} & \multicolumn{1}{c}{63.81\scriptsize{±2.15}} & \multicolumn{1}{c|}{83.33\scriptsize{±9.60}} & \underline{69.26\scriptsize{±3.67}} & 63.64\scriptsize{±2.14} & \underline{60.93\scriptsize{±9.06}} & 76.77\scriptsize{±4.37} & 76.88\scriptsize{±6.30} & \underline{81.40\scriptsize{±7.14}} & \underline{72.68\scriptsize{±0.87}} & \underline{74.47\scriptsize{±0.48}} & \multicolumn{1}{l}{\underline{80.75\scriptsize{±2.05}}} \\
    \midrule
    GNNSafe & 56.85\scriptsize{±8.91} & \multicolumn{1}{c}{56.13\scriptsize{±8.26}} & \multicolumn{1}{c|}{97.00\scriptsize{±3.71}} & \multicolumn{1}{l}{21.93\scriptsize{±1.76}} & 36.88\scriptsize{±1.68} & 95.46\scriptsize{±1.42} & 70.49\scriptsize{±14.80} & 75.67\scriptsize{±15.71} & 87.80\scriptsize{±5.81} & 51.68\scriptsize{±0.08} & 53.97\scriptsize{±0.52} & \multicolumn{1}{l}{95.59\scriptsize{±2.80}} \\
    GraphDE & 61.35\scriptsize{±3.99} & \multicolumn{1}{c}{66.26\scriptsize{±2.98}} & \multicolumn{1}{c|}{99.00\scriptsize{±0.81}} & 66.87\scriptsize{±4.25} & 62.60\scriptsize{±4.47} & 93.06\scriptsize{±8.24} & 26.91\scriptsize{±3.35} & 42.73\scriptsize{±2.06} & 100.00\scriptsize{±0.00} & 59.40\scriptsize{±0.18} & 63.06\scriptsize{±0.30} & \multicolumn{1}{l}{81.82\scriptsize{±0.01}} \\
    GOOD-D & 67.21\scriptsize{±6.41} & \multicolumn{1}{c}{64.86\scriptsize{±6.32}} & \multicolumn{1}{c|}{82.33\scriptsize{±8.31}} & 61.89\scriptsize{±4.87} & \underline{66.91\scriptsize{±7.60}} & 95.20\scriptsize{±4.55} & 52.58\scriptsize{±10.21} & 55.69\scriptsize{±10.56} & 99.20\scriptsize{±1.00} & 56.11\scriptsize{±0.10} & 59.56\scriptsize{±0.16} & \multicolumn{1}{l}{93.67\scriptsize{±0.34}} \\
    AAGOD & 69.25\scriptsize{±4.65} & {65.02\scriptsize{±4.41}} & {82.78\scriptsize{±2.83}} & 70.76\scriptsize{±5.48} & 68.15\scriptsize{±4.45} & {81.56\scriptsize{±22.32}} & 72.51\scriptsize{±1.11} & 67.86\scriptsize{±4.79} & {86.33\scriptsize{±4.03}} & 60.25\scriptsize{±2.16} & {61.44\scriptsize{±1.61}} & {92.53\scriptsize{±1.55}} \\
    \midrule
    OCGIN & 68.11\scriptsize{±4.61} & \multicolumn{1}{c}{68.90\scriptsize{±4.19}} & \multicolumn{1}{c|}{89.67\scriptsize{±3.70}} & 47.51\scriptsize{±9.47} & 50.76\scriptsize{±4.53} & 98.27\scriptsize{±17.70} & 60.78\scriptsize{±5.21} & 57.80\scriptsize{±5.10} & 8780\scriptsize{±9.15} & 59.33\scriptsize{±1.26} & 60.02\scriptsize{±1.88} & \multicolumn{1}{l}{90.00\scriptsize{±2.01}} \\
    GLocalKD & 71.46\scriptsize{±3.21} & \multicolumn{1}{c}{64.93\scriptsize{±4.44}} & \multicolumn{1}{c|}{\underline{78.67\scriptsize{±6.37}}} & 19.82\scriptsize{±1.57} & 35.39\scriptsize{±0.49} & 98.27\scriptsize{±1.13} & \underline{79.39\scriptsize{±4.71}} & \textbf{85.56\scriptsize{±3.33}} & 87.40\scriptsize{±5.42} & 49.60\scriptsize{±1.06} & 51.75\scriptsize{±0.72} & \multicolumn{1}{l}{97.60\scriptsize{±0.35}} \\
    OGGTL & \underline{73.22\scriptsize{±1.83}} & \multicolumn{1}{c}{\textbf{73.61\scriptsize{±3.19}}} & \multicolumn{1}{c|}{82.33\scriptsize{±2.70}} & 54.07\scriptsize{±12.93} & 58.20\scriptsize{±7.86} & 86.40\scriptsize{±6.49} & 37.39\scriptsize{±18.82} & 47.11\scriptsize{±14.06} & 98.80\scriptsize{±2.40} & 51.62\scriptsize{±0.019} & 53.33\scriptsize{±0.01} & \multicolumn{1}{l}{96.79\scriptsize{±0.06}} \\
    \midrule
    SGOOD & \textbf{74.40\scriptsize{±1.42}} & \multicolumn{1}{c}{\underline{72.53\scriptsize{±2.51}}} & \multicolumn{1}{c|}{\textbf{73.66\scriptsize{±7.03}}} & \textbf{78.84\scriptsize{±2.00}} & \textbf{72.54\scriptsize{±3.21}} & \textbf{45.46\scriptsize{±6.62}} & \textbf{80.41\scriptsize{±3.16}} & \underline{83.49\scriptsize{±3.59}} & \textbf{81.20\scriptsize{±2.28}} & \textbf{74.95\scriptsize{±0.79}} & \textbf{74.93\scriptsize{±0.93}} & \multicolumn{1}{l}{\textbf{75.17\scriptsize{±2.72}}} \\
    \midrule
    \vspace{-10pt}
          &       &       & \multicolumn{1}{r}{} &       &       & \multicolumn{1}{r}{} &       &       & \multicolumn{1}{r}{} &       &       & \multicolumn{1}{r}{} \\
    \midrule
    \multirow{2}[2]{*}{Method} & \multicolumn{3}{c|}{BACE} & \multicolumn{3}{c|}{BBBP} & \multicolumn{3}{c|}{DrugOOD} & \multicolumn{3}{c}{HIV} \\
          & AUROC$\uparrow$ & \multicolumn{1}{c}{AUPR$\uparrow$} & \multicolumn{1}{c|}{FPR95$\downarrow$} & AUROC$\uparrow$ & AUPR$\uparrow$ & FPR95$\downarrow$ & AUROC$\uparrow$ & AUPR$\uparrow$ & FPR95$\downarrow$ & AUROC$\uparrow$ & \multicolumn{1}{c}{AUPR$\uparrow$} & \multicolumn{1}{c}{FPR95$\downarrow$} \\
    \midrule
    MSP   & 46.34\scriptsize{±6.10} & 48.65\scriptsize{±3.08} & 97.02\scriptsize{±2.18} & 57.37\scriptsize{±4.28} & 56.84\scriptsize{±3.36} & 94.63\scriptsize{±2.26} & 52.86\scriptsize{±5.26} & 54.49\scriptsize{±4.33} & 98.80\scriptsize{±0.01} & 50.75\scriptsize{±1.88} & 50.49\scriptsize{±0.91} & 95.52\scriptsize{±0.50} \\
    Energy & 46.05\scriptsize{±6.66} & 49.68\scriptsize{±4.16} & 97.36\scriptsize{±2.92} & 56.56\scriptsize{±4.16} & 55.74\scriptsize{±2.78} & 92.68\scriptsize{±2.62} & 52.81\scriptsize{±5.36} & 54.98\scriptsize{±4.36} & 98.20\scriptsize{±1.16} & 50.97\scriptsize{±2.13} & 50.49\scriptsize{±0.91} & 95.50\scriptsize{±0.59} \\
    ODIN  & 45.51\scriptsize{±3.85} & 48.28\scriptsize{±3.76} & 97.02\scriptsize{±1.53} & 54.78\scriptsize{±3.46} & 54.63\scriptsize{±3.69} & 96.34\scriptsize{±1.80} & 51.09\scriptsize{±3.79} & 52.70\scriptsize{±2.66} & 99.00\scriptsize{±1.09} & 50.16\scriptsize{±0.73} & 49.95\scriptsize{±0.58} & 94.60\scriptsize{±1.07} \\
    \md & 73.78\scriptsize{±1.97} & 75.33\scriptsize{±2.32} & 86.78\scriptsize{±6.32} & 53.77\scriptsize{±4.27} & 52.57\scriptsize{±3.81} & 93.29\scriptsize{±2.51} & 66.90\scriptsize{±4.14} & 64.30\scriptsize{±4.43} & \underline{81.60\scriptsize{±4.58}} & \underline{58.10\scriptsize{±3.60}} & \underline{57.18\scriptsize{±3.18}} & \underline{91.89\scriptsize{±1.32}} \\
    \midrule
    GNNSafe & 47.61\scriptsize{±7.50} & 51.52\scriptsize{±5.91} & 98.18\scriptsize{±2.05} & 47.04\scriptsize{±2.40} & 51.52\scriptsize{±5.90} & 98.41\scriptsize{±0.99} & 50.44\scriptsize{±0.57} & \multicolumn{1}{l}{51.14\scriptsize{±0.30}} & \multicolumn{1}{l|}{96.01\scriptsize{±0.33}} & 50.98\scriptsize{±6.82} & \multicolumn{1}{c}{55.13\scriptsize{±6.81}} & \multicolumn{1}{c}{96.01\scriptsize{±0.33}} \\
    GraphDE & 47.32\scriptsize{±1.52} & 51.1\scriptsize{±2.57} & 94.21\scriptsize{±4.58} & 50.88\scriptsize{±2.78} & 51.47\scriptsize{±3.84} & 94.63\scriptsize{±2.34} & 60.19\scriptsize{±4.32} & 62.59\scriptsize{±2.47} & 88.80\scriptsize{±5.60} & 52.38\scriptsize{±1.86} & 54.14\scriptsize{±3.21} & 94.89\scriptsize{±0.84} \\
    GOOD-D & 70.42\scriptsize{±2.22} & 73.21\scriptsize{±3.34} & 88.26\scriptsize{±1.78} & 54.15\scriptsize{±1.10} & 58.58\scriptsize{±1.93} & 99.39\scriptsize{±0.41} & 60.52\scriptsize{±3.33} & 63.09\scriptsize{±2.54} & 98.40\scriptsize{±1.27} & 59.69\scriptsize{±0.62} & 57.10\scriptsize{±.14} & 92.03\scriptsize{±0.61} \\
    AAGOD & 71.41\scriptsize{±2.37} & {71.82\scriptsize{±1.72}} & {90.63\scriptsize{±1.95}} & 58.16\scriptsize{±1.54} & 59.35\scriptsize{±1.32} & {93.5\scriptsize{±0.76}} & 60.29\scriptsize{±3.23} & 66.2\scriptsize{±3.36} & {95.33\scriptsize{±0.47}} & 55.72\scriptsize{±0.69} & {54.29\scriptsize{±0.51}} & {92.2\scriptsize{±0.3}} \\
    \midrule
    OCGIN & 59.71\scriptsize{±5.20} & 61.43\scriptsize{±5.18} & 93.39\scriptsize{±4.44} & 47.78\scriptsize{±5.72} & 47.27\scriptsize{±2.98} & 94.76\scriptsize{±2.70} & 57.95\scriptsize{±5.80} & 59.50\scriptsize{±7.00} & 94.20\scriptsize{±3.12} & 54.06\scriptsize{±0.47} & 52.14\scriptsize{±0.26} & 92.81\scriptsize{±1.01} \\
    GLocalKD & 45.34\scriptsize{±2.11} & 55.39\scriptsize{±2.35} & 98.68\scriptsize{±1.11} & 43.77\scriptsize{±2.23} & 45.84\scriptsize{±1.20} & 98.29\scriptsize{±1.00} & 45.72\scriptsize{±0.97} & 50.90\scriptsize{±3.33} & 100.00\scriptsize{±0.00} & 46.81\scriptsize{±2.90} & 46.95\scriptsize{±2.01} & 97.05\scriptsize{±0.19} \\
    OGGTL & \underline{80.84\scriptsize{±2.00}} & \underline{79.93\scriptsize{±1.26}} & \underline{66.44\scriptsize{±8.89}} & \underline{58.73\scriptsize{±2.19}} & \textbf{60.47\scriptsize{±1.38}} & \underline{91.46\scriptsize{±2.21}} & \underline{67.59\scriptsize{±7.93}} & \multicolumn{1}{l}{\underline{70.90\scriptsize{±5.80}}} & 83.00\scriptsize{±11.22} & 51.78\scriptsize{±0.19} & 53.71\scriptsize{±0.22} & 96.41\scriptsize{±0.05} \\
    \midrule
    SGOOD & \textbf{84.39\scriptsize{±2.73}} & \textbf{83.32\scriptsize{±2.49}} & \textbf{64.13\scriptsize{±4.83}} & \textbf{61.25\scriptsize{±1.60}} & \underline{59.36\scriptsize{±2.39}} & \textbf{88.04\scriptsize{±3.44}} & \textbf{73.15\scriptsize{±4.48}} & \textbf{73.25\scriptsize{±4.49}} & \textbf{67.40\scriptsize{±5.16}} & \textbf{60.82\scriptsize{±0.75}} & \textbf{59.99\scriptsize{±0.69}} & \textbf{90.39\scriptsize{±1.04}} \\
    \bottomrule
    \end{tabular}
    }
  \label{tab:overall}
  \vspace{-1mm}
\end{table*}

\subsection{Experimental Setup}\label{sec::setup}

\subsubsection{Datasets.} We adopt real-world datasets that encompass diverse types of OOD graphs, as listed in Table \ref{tab:datasets}. These datasets are curated from mainstream graph classification benchmarks~\cite{morris2020tudataset,wu2018moleculenet,ji2022drugood} into OOD detection scenarios.
The OOD graph data is generated following \cite{liu2023good, li2022graphde}.
All ID graphs are randomly split into training, validation, and testing with ratio 8:1:1, following the settings of standard graph classification \cite{ hu2020open,morris2020tudataset}. 
The testing set consists of the same number of ID graphs and  OOD graphs. 

\begin{itemize}[leftmargin=*]
    \item    \textbf{\enzymes} comprises protein networks representing enzymes. 'Non-enzymes' protein networks from the PROTEINS dataset \cite{morris2020tudataset} are introduced as OOD.
    
    \item \textbf{\multi}  includes social networks categorized into three classes. We designate social networks from the \bina dataset \cite{morris2020tudataset}, labeled with `Action', as OOD graphs, noting that the `Action' class is absent in \multi.
    \item \textbf{\bina} regards graphs labeled as `Comedy' and `Sci-Fi' in \multi as OOD graphs, as these classes are absent in \bina.
    \item \textbf{\reddit} is a large-scale social network dataset \cite{yanardag2015deep}. Graphs in the dataset 
 REDDIT-BINARY \cite{yanardag2015deep}, which are social networks with classes different from \reddit, are introduced as OOD.

    \item \textbf{\bace} consists of molecules for property prediction. In our adaptation for OOD detection, we use the training set of the original \bace as ID graphs. OOD graphs are introduced by incorporating graphs from the provided test set, wherein molecules exhibit scaffolds distinct from those present in training set.
    \item \textbf{\bbbp}  is constructed similarly to BACE. OOD graphs have molecular scaffolds different from ID graphs.
    \item \textbf{\drug} is a curated OOD dataset consisting of various molecular graphs. We use the provided curator to generate both ID and OOD graphs that have different protein targets.
    \item \textbf{\hiv} is a large-scale molecular graph dataset  \cite{wu2018moleculenet}. Graphs with scaffolds different from ID graphs are regarded as OOD.
\end{itemize}

\subsubsection{Baselines.}
We compare  with 11 competitors in 3 categories.

\begin{itemize}[leftmargin=*]
    \item \textbf{General OOD detection methods}, including \msp~\cite{hendrycks2016baseline}, \energy~\cite{liu2020energy},  \odin~\cite{liang2017enhancing}, and \md~\cite{lee2018simple}. \msp, \energy, and \odin estimate OOD scores directly from classification logits at test time. \msp  uses the maximum softmax score as OOD score while \energy uses energy function. \odin combines temperature scaling with gradient-based input perturbations to enlarge the differences between OOD and ID samples. \md measures Mahalanobis distance between test samples and ID training data.
    \item \textbf{Graph-level OOD detection methods}, including  \safe~\cite{wu2023energy}, \graphde \cite{li2022graphde}, \good \cite{liu2023good} and \aagod \cite{guo2023data}. \safe incorporates GNNs in the energy model and detects OOD samples using energy scores. In our paper, we use graph labels to directly run the basic version of \safe from its Section 3.1~\cite{wu2023energy}. \graphde is a probabilistic model-based approach developed for debiased learning and OOD detection in graph data. 
    \good is an unsupervised OOD detection method that adopts contrastive learning to capture latent patterns of ID graphs.
    \aagod is a post-hoc framework that adopts an adaptive amplifier to enlarge the gap between OOD and ID graphs.
    
    \item \textbf{Graph-level anomaly detection methods}, 
    including \ocgin\cite{zhao2021using}, \ocgtl \cite{qiu2022raising}, and \glkd \cite{ma2022deep}. 
   \ocgin combines deep one-class classification with GIN~\cite{xu2018powerful} to detect outlier graphs at test time. 
   \ocgtl develops a one-class objective for graph anomaly detection.
   \glkd leverages knowledge distillation to detect both local and global graph anomalies.

\end{itemize}

\subsubsection{Evaluation and Implementation.} 
Following ~\cite{wu2023energy,liu2023good}, 
all methods are trained using ID training set and evaluated their OOD detection performance and ID classification performance in the test set. Hyperparameters are tuned using ID graphs from the validation set. 
All methods are evaluated five times on each dataset, and the reported performance metrics are based on the mean and standard deviation results on the test set.
We use three commonly used metrics  \auc, \pr and \fpr for OOD detection evaluation \cite{hendrycks2016baseline,wu2023energy}.  For the classification performance in ID graphs, we use Accuracy (\acc).
Remark that the priority of the graph-level OOD detection task is to accurately identify OOD graphs, instead of improving the ID ACC. 
For \algo, we set the number of layers $L_1=3$ and $L_2=2$, and dimension $d$ as 16.  We set batch size $128$. Training consists of $100$ epochs for pre-training ($T_{PT}$) and $500$ epochs for fine-tuning ($T_{FT}$). We experiment with learning rates in the range $\{0.01, 0.001, 0.0001\}$ for the initial stage, and set the learning rate to $0.001$ and $\alpha$ to $0.1$ for the refinement stage. For all baselines, we use the suggested parameters from their papers or obtained by grid search. Experiments are conducted on a Linux server equipped with an Nvidia RTX 3090 GPU card.

\subsection{Overall Performance}

\header
\textbf{OOD detection performance.}
Table \ref{tab:overall} reports the overall graph-level OOD detection performance of all methods by \auc, \pr and \fpr metrics on all datasets, by mean and standard deviation values. 
Observe that \algo consistently achieves superior OOD detection effectiveness under most settings. 
For instance, on \multi, \algo has \auc 78.84\%, which indicates $9.58\%$ absolute improvement over the best competitor with \auc 69.26\%. As another example on \bace molecule dataset, the \auc of \algo is 84.39\%, while the runner-up achieves \auc 80.84\%.
The overall results in Table \ref{tab:overall} show that \algo  effectively encode task-agnostic substructures into expressive representations for graph-level OOD detection, validating the power of our technical designs.

\begin{table}[!t]
  \caption{ ID graph classification performance measured by average \acc (in percentage \%).  {/ indicates that \acc is not applicable for  unsupervised  methods.} }  \vspace{-2mm}
  \setlength{\tabcolsep}{1.2pt}
  \resizebox{0.48\textwidth}{!}{
    \begin{tabular}{ccccccccc}
    \toprule
    \multicolumn{1}{l}{Method} & \multicolumn{1}{l}{ENZYMES} & \multicolumn{1}{l}{IMDB-M} & \multicolumn{1}{l}{IMDB-B} & \multicolumn{1}{l}{\reddit} & \multicolumn{1}{l}{BACE} & \multicolumn{1}{l}{BBBP} & \multicolumn{1}{l}{HIV} & \multicolumn{1}{l}{DrugOOD} \\
    \midrule
    MSP   & 37.33 & 48.27 & 69.80 & 48.91 & 80.83 & 87.44 & 96.62 & 79.20 \\
    Energy & 37.33 & 48.27 & 69.80 & 48.91 & 80.83 & 87.44 & 96.62 & 79.20 \\
    ODIN  & 37.33 & 48.27 & 69.80 & 48.91 & 80.83 & 87.44 & 96.62 & 79.20 \\
    \md   & 37.33 & 48.27 & 69.80 & 48.91 & 80.83 & 87.44 & 96.62 & 79.20 \\
    \midrule
    GNNSafe & 17.66 & 30.13 & 50.20 & 27.42 & 56.69 & 79.14 & 96.58 & 64.40 \\
    GraphDE & 46.00 & 37.86 & 69.80 & 40.68 & 77.68 & 88.90 & 96.20 & 77.00 \\
    GOOD-D & /     & /     & /     & /     & /     & /     & /     & / \\
    \aagod & /     & /     & /     & /     & /     & /     & /     & / \\
    \midrule
    SGOOD & \textbf{48.66} & \textbf{48.66} & \textbf{71.60} & \textbf{51.82} & 80.33 & \textbf{89.14} & \textbf{96.66} & \textbf{79.40} \\
    \bottomrule
    \end{tabular}
    }
  \label{tab:acc}
  \vspace{-1mm}
\end{table}

\header
\textbf{ID graph classification performance.}
Table~\ref{tab:acc} reports the performanc on ID graph classification of all methods by Accuracy (\acc). The results of graph-level anomaly detection methods are not reported as \acc is not applicable. {\algo  achieves
the best ID ACC on 7 out of 8 datasets. For instance,  on \enzymes, \algo has \acc 48.66\%, while the \acc of the best competitor \graphde is 46.00\%,
indicating a relative improvement of 5.8\%. On \reddit, a large-scale dataset with 11 ID classes, \algo has \acc 51.82\%, outperforming the best competitor by a relative improvement of 5.9\%.}
The results indicate that leveraging substructures benefits both OOD detection and graph classification.

\begin{table}[!t]

  \centering
  \caption{Ablation  \auc (\%)}
  \vspace{-2mm}
    \setlength{\tabcolsep}{1.5pt}

  \resizebox{0.44\textwidth}{!}{
    \begin{tabular}{ccccccc}
    \toprule
    Method &  {\enzymes} &  {\multi} &  {\bina} & {\bace} &  {\bbbp} &  {\drug} \\
    \midrule
         Best baseline  & 	71.46 & 69.26  & \underline{79.39} & 73.78  & 57.37  & 57.37  \\
         \midrule
          \algo (base)  &  67.38  & 69.26  & 76.80  & 73.78  & 53.77  & 66.90  \\
         $\text{\algo}\backslash\text{A}$  & \underline{73.60}  & \underline{75.22}  & 77.80  & \underline{75.96}  & \underline{57.84}  & \underline{68.80} \\
    \algo &\textbf{74.41} & \textbf{78.84} & \textbf{80.42} & \textbf{84.40} & \textbf{61.25} & \textbf{73.16} \\
    \bottomrule
    \end{tabular}}
    
  \label{tab:ablation}
  \vspace{-2mm}
\end{table}

\begin{figure}[!t]
    
    \label{fig:enter-label}
      
      \includegraphics[width=0.96\linewidth]{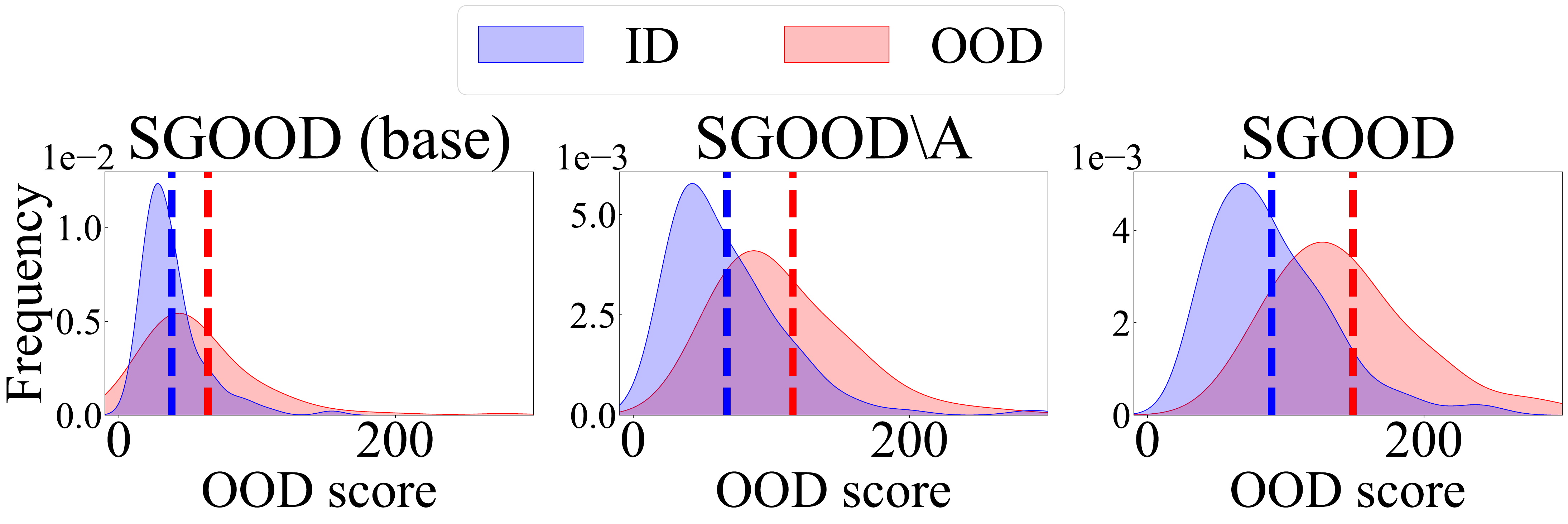}
    
    \vspace{-2mm}
    \caption{\label{fig:distribution}ID and OOD score distributions, with the dotted line indicating the mean of ID/OOD scores. }
    
     \vspace{-2mm}
\end{figure}

\subsection{Model Analysis}

\header
\textbf{Ablation Study.}
In Table \ref{tab:ablation},   $\text{\algo}\backslash\text{A}$ is \algo that ablates all augmentations in Section \ref{sec::augmentations}, \ie $\alpha$=$0$ in Eq. \eqref{eq:loss}; \algo (base) further ablates all substructure-related representations in Section \ref{sec::substructure}.
In Table \ref{tab:ablation}, first observe that, from \algo (base) to $\text{\algo}\backslash\text{A}$ and then to the complete version \algo,
the performance gradually increases on all datasets, validating the effectiveness of all proposed techniques.
Second,  $\text{\algo}\backslash\text{A}$  already surpasses the best baseline performance on most datasets, which demonstrates the effect of the techniques in Section \ref{sec::substructure}, without the augmentation techniques in Section \ref{sec::augmentations}.
With the help of the substructure-preserving graph augmentations, \algo pushes the performance further higher.  
In Figure \ref{fig:distribution}, we visualize the ID and OOD score distributions of \algo (base),  $\text{\algo}\backslash\text{A}$ and \algo on \drug, with their mean scores shown in dotted lines. Clearly, we are obtaining more separable OOD scores against ID data from left to right in Figure \ref{fig:distribution}, which demonstrates that our techniques in \algo can learn distinguishable representations for ID and OOD graphs.

\begin{table}[!t]

  \centering
  \caption{Comparison between different substructure detection methods by \auc (\%).}  \vspace{-2mm}

  \setlength{\tabcolsep}{1.5pt}

  \resizebox{0.47\textwidth}{!}{
    \begin{tabular}{lcccccc}
    \toprule
    
    \algo & \enzymes & \multi & \bina & \bace  & \bbbp  & \drug \\
    \midrule
    \emph{w.o.} substructures &   67.38 & 69.26 & 76.8  & 73.78 & 53.77 & 66.90 \\  
    Modularity  & \textbf{74.41 } & \textbf{78.84 } & \textbf{80.42 } & \textbf{84.40 } & \textbf{61.25 } & \textbf{73.16 }  \\
    Graclus & \underline{71.12}  & 74.64  & \underline{78.86}  & \underline{79.54}  & 56.62  & 67.94  \\
    LP & 68.09  & \underline{75.48}  & 78.46 & 76.63  & 54.90  & \underline{68.95}  \\
    BRICS &   /    &    /   &    /  & 78.39  & \underline{60.18}  & 64.78  \\
    \bottomrule
    \end{tabular}}
  \label{tab:community}
  \vspace{-2mm}
\end{table}

\begin{table}[!t]

  \centering
  \caption{Comparing with subgraph-aware models {\small \auc} (\%). \textbf{Bold}: best. \underline{Underline}: runner-up.}
  \vspace{-2mm}
  \setlength{\tabcolsep}{0.6mm}
  \resizebox{0.42\textwidth}{!}{
    \begin{tabular}{ccccccc}
    \toprule
    \multicolumn{1}{l}{Method} & ENZYMES & IMDB-M & IMDB-B & BACE  & BBBP  & \drug \\

    \midrule
          
    SAG   & 70.40  & 76.50  & 77.30  & \underline{76.90 } & \underline{58.90 } & 65.80  \\
    TopK  & 70.20  & \underline{76.80 } & 77.20  & 74.20  & 54.90  & 58.50  \\
    DiffPool & 73.30  & 75.90  & 78.00  & 76.50  & 57.50  & 70.60  \\
    \ngnn  & 70.30  & 71.20  & 76.60  & 71.20  & 52.60  & \textbf{75.60 } \\
    \gnnak & 68.50  & 73.50  & 77.20  & 70.90  & 54.30  & 63.00  \\
    \midrule
    SGOOD & \textbf{74.41 } & \textbf{78.84 } & \textbf{80.42 } & \textbf{84.40 } & \textbf{61.25 } & \underline{73.16 } \\
    \bottomrule
    \end{tabular}
    }
  \label{tab:subgraph}

\vspace{-2mm}
\end{table}

\header
\textbf{Effect of Task-agnostic Substructure Detection Methods.}
As mentioned, \algo is orthogonal to  specific substructure extraction methods.
Here in \algo, we evaluate several commonly-used methods to extract substructures, including  Graclus \cite{dhillon2007weighted}, label propagation (LP) \cite{cordasco2010community}, and 
BRICS \cite{degen2008art}. BRICS  uses chemistry knowledge  for extraction.
In Table \ref{tab:community}, \algo with different substructure detection methods are all better than \algo \emph{w.o.} substructures, and \algo with Modularity is the best. The results validate the effectiveness of \algo that leverages substructures for OOD detection.

\header
\textbf{Comparison with Task-specific Subgraph Models.}
We then compare \algo directly with subgraph GNN models, including three
hierarchical pooling methods (SAG \cite{lee2019self}, TopK \cite{gao2019graph}, DiffPool \cite{ying2018hierarchical}) and  two subgraph GNNs (\ngnn \cite{zhang2021nested} and \gnnak \cite{zhao2021stars}).
Note that these methods are not specifically designed for graph-level OOD detection.
\commentout{ \zhihao{Therefore, we train them with: 1) classification loss without contrastive learning; 2) classification loss empowered with contrastive learning and graph augmentations \cite{ ding2022data,you2020graph} following the same procedure as \algo.}}
At test time, we extract the graph representations generated by these methods and use Mahalanobis distance as OOD score.
Table \ref{tab:subgraph} reports the \auc results.
\algo performs best on 5 out of 6 datasets and is the top-2 on DrugOOD. This validates the effectiveness of our substructure-related techniques in Section \ref{sec:method} for graph-level OOD detection.

\header
\textbf{Effect of Augmentations.} We evaluate the augmentations (SD, SG, and SS) in Section \ref{sec::augmentations}, with conventional graph augmentations that are not substructure-preserving, 
  including edge perturbation (EP), attribute masking (AM), node dropping (ND), and subgraph sampling (SA).
Table \ref{tab:augmentations} reports the results, AM is not applicable on \multi and \bina  since they do not have node attributes. 
Observe that our SD, SG, and SS are the top-3 ranked techniques for graph-level OOD detection, validating the effectiveness of the proposed substructure-preserving graph augmentations. 
In Appendix  Figure  \ref{fig:heatmap}, we also visualize the improvements of all pairwise combinations of our augmentation techniques.

\header
\textbf{Performance under Different Backbones.}
We evaluate \algo and competitors when changing the GIN backbone to GCN~\cite{kipf2017semi} and GraphSage~\cite{hamilton2017inductive}. Table~\ref{tab:backbone} reports the results.
With GCN backbone, compared with the baselines, \algo consistently achieves the best scores; with GraphSage backbone, \algo is the best on BACE, BBP, DrugOOD, and the second best on other datasets. The results validate the robustness of \algo to different backbones.

\header
\textbf{Model efficiency.}
We compare the training time per epoch in seconds of all methods, with results in Table~\ref{tab:time}. 
Compared with other graph-level OOD detection competitors, including \graphde, \good, and \aagod, \algo requires less time to train. 
Compared with all methods, including the methods originally designed for image data, \algo requires moderate time for training.
Considering together the time cost 
in Table \ref{tab:time} and the effectiveness in Table \ref{tab:overall}, we can conclude
that \algo is effective and efficient for graph-level OOD detection,

\begin{table}[!t]

\centering
  \caption{Comparing with different augmentations by \auc.}
  \vspace{-2mm}
  \setlength{\tabcolsep}{1.5pt}
  \resizebox{0.44\textwidth}{!}{
    \begin{tabular}{lccccccc}
    \toprule
     & \enzymes & \multi  & \bina & \bace  & \bbbp  & \drug & Avg. Rank \\
    \midrule
    EP & 74.28  & 76.50  & 78.44  & 78.75  & 58.24  & 71.32  & 4.67 \\
    AM & 72.44  & /    & /     & 77.28  & 59.68  & 71.27  & 5.25 \\
    ND & 73.11  & 77.09  & 78.40  & 78.79  & 58.59  & 69.48  & 4.83 \\
    SA & 72.12  & 76.76  & 79.25  & 77.13  & 57.84  & \textbf{72.66} & 4.83 \\
    \midrule
    SD & \textbf{74.77} & \textbf{78.15} & 79.54  & 82.00  & 59.76  & 72.65  & \textbf{1.83} \\
    SG & 72.74  & 77.98  & 78.97  & 82.24  & 59.58  & 71.97  & 3.33 \\
    SS & 74.27  & 76.20  & \textbf{80.50} & \textbf{83.53} & \textbf{63.53} & 71.94  & 2.67 \\
    \bottomrule
    \end{tabular}}
  \label{tab:augmentations}
  \vspace{-2MM}
\end{table}

\begin{table}[!t]
  \centering
  \caption{Performance with different
backbones by AUROC (\%).  \textbf{Bold}: best. \underline{Underline}: runner-up.} 
\vspace{-2mm}
    \setlength{\tabcolsep}{1.5pt}
  \resizebox{0.45\textwidth}{!}{
    \begin{tabular}{cccccccc}
    \toprule
    Backbone & Method & ENZYMES & IMDB-M & IMDB-B & BACE  & BBBP  & DrugOOD \\
    \midrule
    \multirow{5}[2]{*}{GCN} & \md & {\underline{70.04}}  & \underline{71.27}  & 53.46  & \underline{72.68}  & 54.97  & \underline{66.01}  \\
          & GraphDE & 61.40  & 68.44  & 29.13  & 53.24  & 52.50  & 56.61  \\
          & GOOD-D & 41.96  & 61.71  & 59.53  & 72.52  & \underline{58.91}  & 61.79  \\
          & OCGIN & 64.35  & 57.46  & \underline{64.08}  & 67.54  & 51.23  & 59.30  \\
          & SGOOD & \textbf{71.26 } & \textbf{73.52 } & \textbf{65.91 } & \textbf{83.42 } & \textbf{62.76 } & \textbf{72.52 } \\
    \midrule
    \multirow{5}[2]{*}{GraphSage} & \md & 68.07  & 48.06  & 43.63  & \underline{73.60}  & 53.88  & \underline{64.55}  \\
          & GraphDE & 61.37  & \textbf{69.65 } & 28.28  & 53.24  & 52.50  & 56.66  \\
          & GOOD-D & 45.55  & 57.02  & 23.90  & 73.15  & \underline{56.85}  & 61.57  \\
          & OCGIN & \textbf{71.75 } & 36.86  & \textbf{71.44 } & 57.47  & 46.65  & 63.82  \\
          & SGOOD & \underline{70.21}  & \underline{68.63}  & \underline{61.59}  & \textbf{82.22 } & \textbf{59.50 } & \textbf{68.60 } \\
    \bottomrule
    \end{tabular}
    }
  \label{tab:backbone}
  \vspace{-4mm}
\end{table}

\subsection{Parameter Sensitivity}
In this section, we evaluate the performance of \algo under varying hyperparameters to test its robustness and sensitivity to parameter settings.

\commentout{In Table \ref{tab:ablation} reports the performance of \algo with $\mathcal{L}_{CL}$ ablated ($\text{\algo}\backslash\text{A}$), in which the completed \algo with contrastive learning is consistently better for graph-level OOD detection.}
\commentout{\zhihao{\algo (base) indicates the baseline version of \algo  without using techniques proposed by us.}}

\begin{table}[!t]
  \centering
  \caption{Training time per epoch of all methods on all datasets by seconds (s).}
  \setlength{\tabcolsep}{0.9pt}
  \resizebox{0.95\linewidth}{!}{
    \begin{tabular}{ccccccccc}
    \toprule
   Method & ENZYMES & IMDB-M & IMDB-B & REDDIT-12K & BACE  & BBBP  & HIV   & DrugOOD \\
    \midrule
    \msp  & 0.119  & 0.077  & 0.090  & 0.890  & 0.053  & 0.055  & 2.740  & 0.078  \\
    \energy & 0.119  & 0.077  & 0.090  & 0.890  & 0.053  & 0.055  & 2.740  & 0.078  \\
    \odin & 0.119  & 0.077  & 0.090  & 0.890  & 0.053  & 0.055  & 2.740  & 0.078  \\
    \md   & 0.119  & 0.077  & 0.090  & 0.890  & 0.053  & 0.055  & 2.740  & 0.078  \\
    \midrule
    \graphde & 1.692  & 1.175  & 1.392  & 176.400  & 0.950  & 0.696  & 43.770  & 1.020  \\
    \good & 0.257  & 0.171  & 0.197  & 17.550  & 0.157  & 0.095  & 5.160  & 0.230  \\
    \aagod & 0.234  & 0.269  & 0.404 & 5.927 & 0.380  & 0.681 & 14.000 & 0.310  \\
    \midrule
    \ocgin & 0.123  & 0.079  & 0.086  & 1.650  & 0.075  & 0.044  & 2.900  & 0.099  \\
    \glkd & 0.072  & 0.203  & 0.054  & 142.000  & 0.052  & 0.035  & 4.220  & 0.067  \\
    \algo & 0.161  & 0.138  & 0.137  & 0.980  & 0.085  & 0.058  & 3.970  & 0.124  \\
    \bottomrule
    \end{tabular}
    }
  \label{tab:time}
  
\end{table}

\begin{figure}[!t] 

    \centering 
    \includegraphics[width=0.45\textwidth]{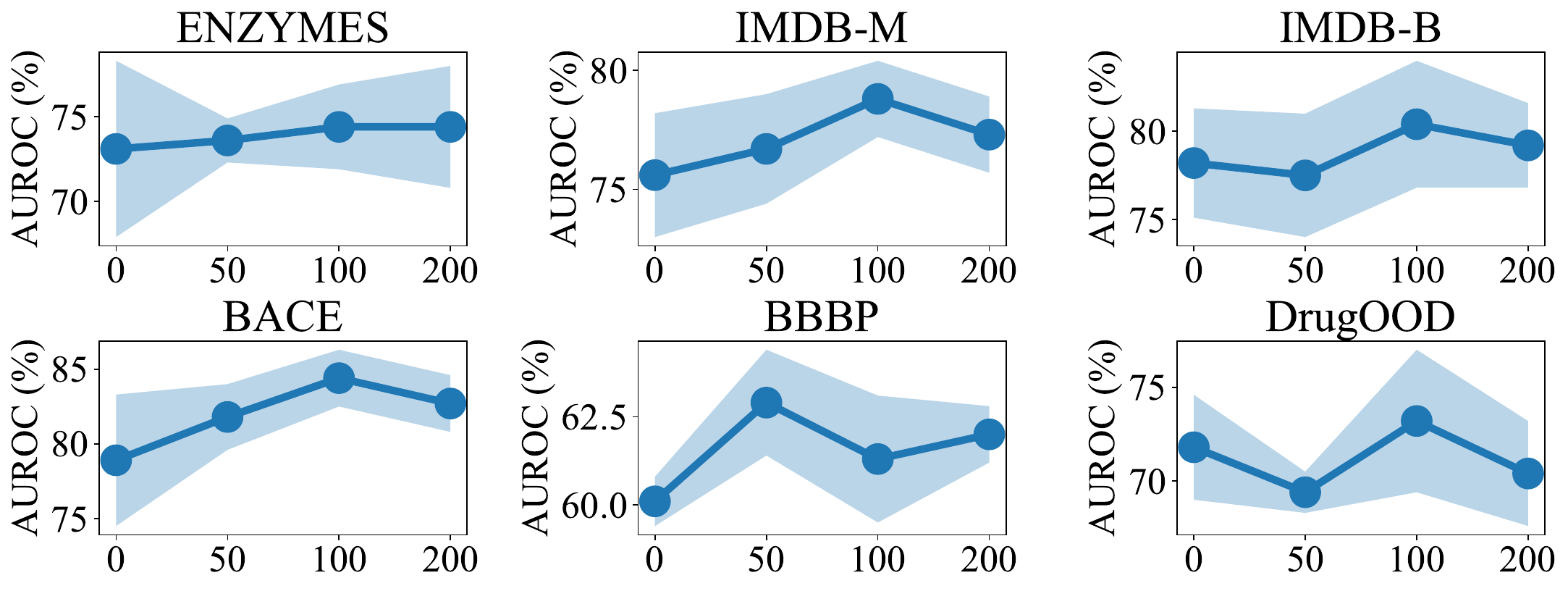}
    
    \caption{OOD detection performance  of \algo by \auc (\%) when the number of pretraining epochs $T_{PT}$ varies from 0 to 200, with colored area representing standard deviation.}
    \label{fig:epochs} 

\end{figure}

\header
\textbf{Varying pretraining epochs $T_{PT}$.}
 We conduct experiments to study the effect of pretraining epochs $T_{PT}$ from 0 to 200. As shown in Figure~\ref{fig:epochs}, compared to \algo without first-stage pretraining ($T_{PT}=0$), pretraining  improves \algo's performance. We also found that excessive pretraining can sometimes have negative effects. For example, when $T_{PT}=200$, \algo's performance decrease  on all datasets except \enzymes. We speculate the reason is that excessive pretraining makes task-agnostic information dominate, with a negative impact on the \algo's ability to learn from class labels. As $T_{PT}=100$ generally leads to competitive performance across all datasets, 
 we set the default value of  $T_{PT}$ to 100.

\begin{figure}[!t] 

    \centering 
    \includegraphics[width=0.45\textwidth]{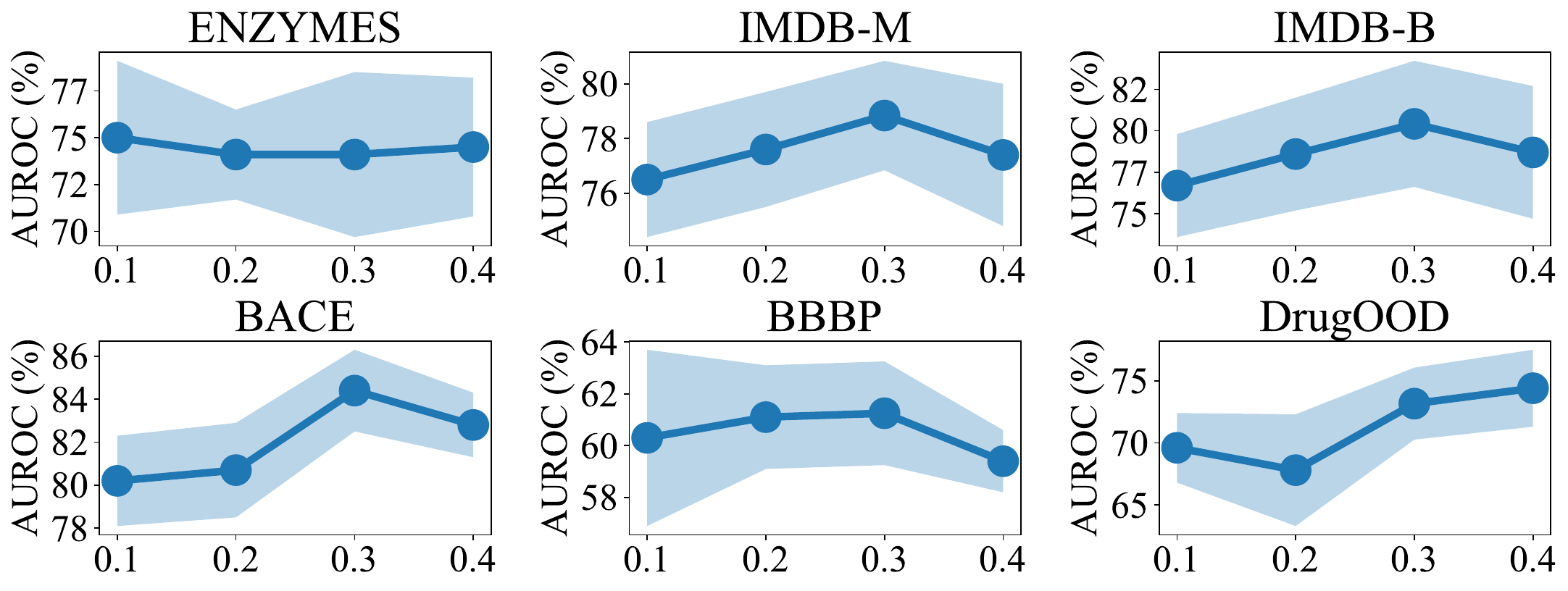}
    
    \caption{OOD detection results of \algo by \auc (\%) when the weight of the contrastive loss $\alpha$ varies from 0 to 1, with the colored area representing standard deviation.}
    \label{fig:weight} 
    \vspace{-2mm}
\end{figure}

\header
\textbf{Varying the weight of contrastive loss  $\alpha$.} We vary $\alpha$ from 0 to 1 to study the effect. As shown in Figure~\ref{fig:weight}, compared to \algo fine-tuned solely by $\mathcal{L}_{CE}$ (\ie $\alpha=0$), fine-tuning \algo with both $\mathcal{L}_{CE}$ and $\mathcal{L}_{CL}$ generally leads to better performance. 
As $\alpha = 0.1$ usually leads to competitive performance across all datasets, we set the default value of  $\alpha$ to 0.1 in \algo.

\begin{figure}[!t] 

    \centering 
    \includegraphics[width=0.45\textwidth]{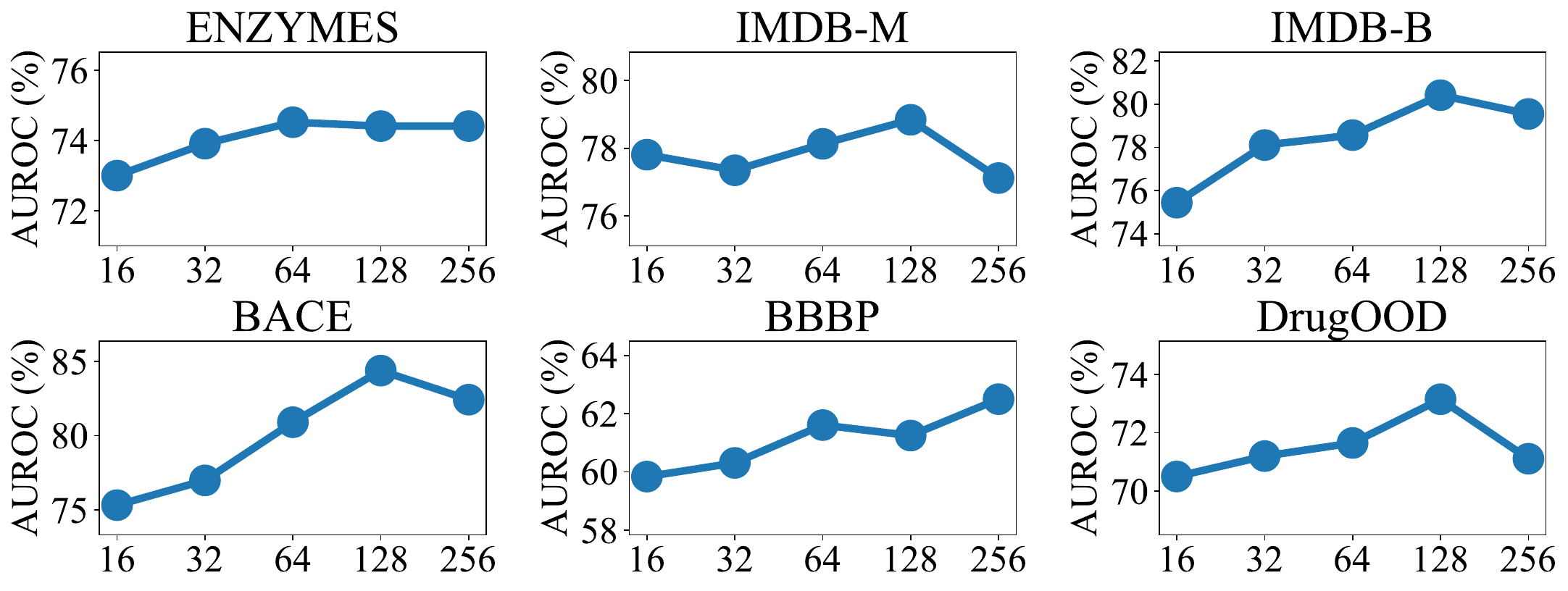}
    \vspace{-2mm}
    \caption{  OOD detection results  of \algo  by \auc (\%) when the batch size $B$ varies from 16 to 256.}
    \label{fig:batch} 
    \vspace{-3mm}
\end{figure}

\begin{figure}[!t] 

    \centering 
    \includegraphics[width=0.45\textwidth]{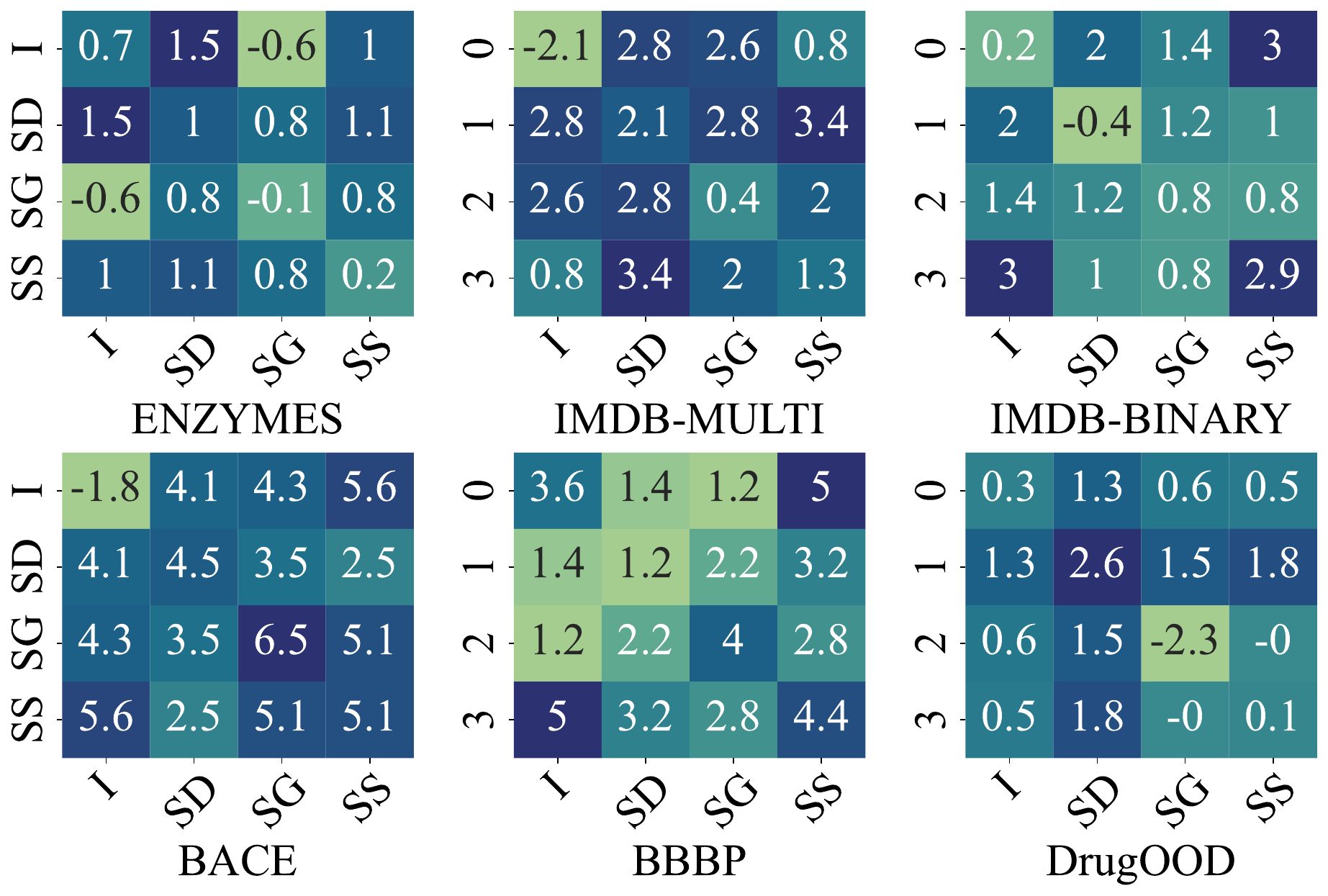}
    
      \vspace{-2mm}
    \caption{\auc gain (\%) of \algo  compared with $\text{\algo}\backslash\text{A}$ without graph augmentations.}
    \label{fig:heatmap} 
    \vspace{-2mm}
\end{figure}

\header
\textbf{Varying number of negative samples in $\mathcal{L}_{CL}$. } We conduct experiments to study the effect of the number of  negative samples used in  contrastive loss $\mathcal{L}_{CL}$ (Eq.\eqref{eq:lossCL}). Following the established convention in graph contrastive learning~\cite{you2020graph}, pairs of augmented graphs originating from the same graph are treated as positive pairs, while pairs generated from different graphs within the batch are considered negative pairs. In such a way, in a $B$-size batch, for every $G_i$, it will have $2B-2$ negative samples, as shown in the denominator of Eq.\eqref{eq:lossCL}. Apparently the number of negative samples is related to batch size $B$. We vary $B$ from 16 to 256 to evaluate sensitivity of SGOOD w.r.t. the number of negative samples, and report the results in  Figure~\ref{fig:batch} . Observe that as increasing from 16 to 128, the overall performance increases and then becomes relatively stable, which proves the effectiveness of the augmentation techniques developed in \algo and also validates the superior performance of \algo when varying  batch size and the number of negative samples.

\begin{table}[!t]
\centering
  \caption{Varying $L_{1}$ and $L_{2}$ in \algo (\auc).}
    \vspace{-1mm}
    \setlength{\tabcolsep}{3pt}

  \resizebox{0.45\textwidth}{!}{
    \begin{tabular}{ccccccccc}
    \toprule
    $L_{1}$ & $L_{2}$ & \enzymes & \multi & \bina & \bbbp  & \bace  & \drug \\
    \midrule
    4     & 1     & 74.00  & 77.13  & \textbf{81.00 } & 80.43  & 62.00  & 71.17  \\
    3     & 2     & \textbf{74.41 } & \textbf{78.84 } & 80.42  & \textbf{84.40 } & 61.25  & \textbf{73.16 } \\
    2     & 3     & 73.63  & 76.03  & 79.05  & 80.34  & \textbf{62.26 } & 69.12  \\
    1     & 4     & 74.22  & 77.83  & 76.79  & 76.62  & 61.08  & 68.01  \\
    \bottomrule
    \end{tabular}
  \label{tab:layers}
  }
\vspace{-2mm}
\end{table}

\header
\textbf{Varying different combinations of augmentations.}
In Figure \ref{fig:heatmap}, we exhaust the pairwise combinations of all options in $\mathcal{A}=\{\text{\underline{I, SD, SG, SS}}\}$ and visualize the \auc gain on graph-level OOD detection over $\text{\algo}\backslash\text{A}$ without graph augmentations. As shown in Figure \ref{fig:heatmap}, most combinations achieve positive gains for effective OOD detection.

\header
\textbf{Varying $L_1$ and $L2$.} 
In the experiments above, we fix the layers of the two GINs in the two-level graph encoding in Section \ref{sec::substructure} to be $L_1=3$ and $L_2=2$ as default. If we search $L_1$ and $L_2$, it is possible to get even better OOD detection results, as shown  in Table \ref{tab:layers} where $L_1$ and $L_2$ are varied with their sum fixed to be 5. For example, on BACE with $L_1$=$2$ and $L_2$=$3$, \auc is 62.26\%, about 1\% higher than the default setting.

\section{Related Work}\label{sec:related}

\header \textbf{Graph-level Representation Learning.}
Graph-level representation learning aims to learn representations of entire graphs \cite{wu2020comprehensive}.  GNNs \cite{hamilton2017inductive,kipf2017semi,velivckovic2017graph,xu2018powerful} are often adopted  \cite{guo2022graph, yang2022learning} to first learn node representations by   message passing on graphs, and then node representations are aggregated by flat pooling functions to get graph-level representations \cite{xu2018powerful}.
{However, these traditional methods have limitations in capturing high-order structures with crucial semantics for graph-level tasks, \eg functional groups in molecules~\cite{ying2018hierarchical}.
Hence, there exist methods to leverage subgraphs}, \eg hierarchical pooling \cite{gao2019graph,lee2019self,ying2018hierarchical} and subgraph GNNs \cite{zhang2021nested,zhao2021stars}. 
Hierarchical pooling methods learn to assign nodes into different clusters and coarsen graphs hierarchically. Subgraph GNNs apply message passing on extracted rooted-subgraphs of nodes in a graph, and then aggregate subgraph representations \cite{frasca2022understanding}.
These methods are not designed for graph OOD detection, and they assume that graphs are i.i.d in training and testing and learn task-specific substructures. In this paper, we explore leveraging task-agnostic substructures to learn expressive graph representations for OOD detection.

\header \textbf{Graph Out-of-distribution Detection.}
{Out-of-Distribution (OOD) detection has recently received considerable research attention on graph data.}
\cite{wu2023energy} explore node-level OOD detection by using energy function to detect  OOD nodes in a graph, which is a different problem from this paper. For graph-level OOD detection, \cite{li2022graphde} design a generative model that has the ability to identify outliers in training graph samples, as well as OOD samples during the testing stage.   \cite{liu2023good} develop a self-supervised learning approach to train their model to estimate OOD scores at test time. 
{Recently, \cite{zhang2022dual} proposes to learn anomalous substructures using deep random walk kernel, which depends on labeled anomalous graphs, while OOD graphs are unseen during the training stage and only available during the testing stage.} Instead of training GNNs for OOD detection, \aagod~\cite{guo2023data} develops an adaptive amplifier that modifies the graph structure to enlarge the gap between OOD and ID graphs.
Observe that existing graph-level OOD detection methods mainly leverage node representations output by GNNs \cite{kipf2017semi, velivckovic2017graph,xu2018powerful,ZhouSZH023} to get graph-level representations, while the rich substructure patterns hidden in graphs are under-investigated for graph-level OOD detection. 
On the other hand, our method \algo explicitly uses substructures in graphs to learn high-quality representations for effective graph-level OOD detection.

\section{Conclusion}
We study the problem of graph-level OOD detection, and present a novel \algo method with superior performance.
The design of \algo is 
motivated by the exciting finding that substructure differences commonly exist between ID and OOD graphs, and \algo aims to preserve more distinguishable graph-level representations between ID and OOD graphs.
Specifically, we build a super graph of substructures for every graph, and develop a two-level graph encoding pipeline to obtain high-quality structure-enhanced graph representations.
We further develop a set of substructure-preserving graph augmentations. 
Extensive experiments on real-world graph datasets with various OOD types validate the superior performance of \algo over existing methods for graph-level OOD detection.
\bibliographystyle{ACM-Reference-Format}
\bibliography{biblio}

\end{document}